\documentclass[lettersize,journal]{IEEEtran}
\usepackage{amsmath,amsfonts}
\usepackage{algorithmic}
\usepackage{algorithm}
\usepackage{array}
\usepackage[caption=false,font=normalsize,labelfont=sf,textfont=sf]{subfig}
\usepackage{textcomp}
\usepackage{stfloats}
\usepackage{url}
\usepackage{verbatim}
\usepackage{graphicx}
\usepackage{cite}

\usepackage{mathrsfs}
\usepackage{tabularx}
\usepackage{diagbox}
\usepackage{epstopdf}
\usepackage{pdfpages}
\usepackage{multirow}
\usepackage{amsthm}
\theoremstyle{remark}

\theoremstyle{plain}
\newtheorem{theorem}{Theorem}
\newtheorem{proposition}{Proposition}

\newtheorem{assumption}{Assumption}

\usepackage{amssymb}
\usepackage{placeins}
\usepackage{glossaries}

\newacronym{LSTM}{LSTM}{long short-term memory}

\newacronym{MLP}{MLP}{multi-layer perceptron}

\newacronym{IR}{IR}{integrity risk}

\newacronym{PL}{PL}{protection level}

\newacronym{PE}{PE}{position error}

\newacronym{AL}{AL}{alert limit}

\newacronym{GNSS}{GNSS}{global navigation satellite systems}

\newacronym{LOS}{LOS}{line-of-sight}

\newacronym{NLOS}{NLOS}{non-line-of-sight}

\newacronym{QL}{QL}{quantile loss}

\newacronym{HQL}{HQL}{Huber quantile loss}

\newacronym{CDF}{CDF}{cumulative distribution function}

\newacronym{PDF}{PDF}{probability density function}

\newacronym{CN0}{CN0}{carrier-to-noise-density ratio}

\newacronym{CP}{CP}{conformal prediction}

\newacronym{CQR}{CQR}{conformalized quantile regression}

\newacronym{QR}{QR}{quantile regression}

\newacronym{PRN}{PRN}{pseudorandom noise}

\newacronym{MPN}{MPN}{multipath errors and noise}

\newacronym{CMC}{CMC}{code-minus-carrier}

\newacronym{RNN}{RNN}{recurrent neural network}

\newacronym{CNN}{CNN}{convolutional neural network}

\newacronym{GNN}{GNN}{graph neural network}

\newacronym{EMD}{EMD}{Earth Mover’s Distance}

\newacronym{QQ}{QQ}{quantile-quantile}

\newacronym{ZWD}{ZWD}{Zenith Wet Delay}

\newacronym{OOD}{OOD}{out-of-distribution}

\newacronym{SBAS}{SBAS}{Satellite-Based Augmentation System}

\newacronym{GBAS}{GBAS}{Ground-Based Augmentation System}

\newacronym{ARAIM}{ARAIM}{Advanced Receiver Autonomous Integrity Monitoring}

\newacronym{GMM}{GMM}{Gaussian Mixture Model}

\newacronym{GPR}{GPR}{Gaussian Process Regression}

\newacronym{CDDIS}{CDDIS}{Crustal Dynamics Data Information System}

\newacronym{IGS}{IGS}{International GNSS Service}

\newacronym{PPP}{PPP}{Precise Point Positioning}

\newacronym{GIVE}{GIVE}{Grid Ionospheric Vertical Error}

\newacronym{Rfit}{Rfit}{fit radius}

\newacronym{RCM}{RCM}{relative centroid metric}

\newacronym{TEC}{TEC}{Total Electron Content}

\newacronym{UAV}{UAV}{unmanned aerial vehicle}

\newacronym{AD}{AD}{autonomous driving}

\newacronym{UAM}{UAM}{urban air mobility}

\newacronym{UQ}{UQ}{uncertainty quantification}

\graphicspath{{images/}}

\begin{document}

\title{Learning Context-conditioned Gaussian Overbounds for Convolution-Based Uncertainty Propagation}

\author{Ruirui Liu, Xuejie Hou, Yiping Jiang, Hui Ren
\thanks{}%
\thanks{Ruirui Liu, Xuejie Hou, Yiping Jiang, and Hui Ren are with the Department of Aeronautical and Aviation Engineering, The Hong Kong Polytechnic University, Hong Kong, China (email: ruirui.liu@connect.polyu.hk; yiping.jiang@polyu.edu.hk).}}

\markboth{}%
{Liu \MakeLowercase{\textit{et al.}}: Learning Context-conditioned Gaussian Overbounds for Convolution-Based Uncertainty Propagation}

\maketitle

\begin{abstract}
Uncertainty quantification is essential in safety-critical settings—from autonomous driving to aviation, finance, and health—where decisions must rely on conservative bounds rather than point estimates. Predictor-level intervals (e.g., from quantile regression, conformal prediction, variance networks, or Bayesian models) generally do not compose: adding two per-variable intervals need not yield a valid interval for their sum or preserve coverage. In aviation, Gaussian overbounding replaces complex error distributions with a conservative Gaussian whose tails dominate the truth, so conservatism propagates through linear operations. Yet classical overbounds are global, often overly conservative, and hard to adapt to feature-conditioned errors. We propose a unified learning framework that trains neural networks to produce context-aware Gaussian overbounds—mean and scale—with provable conservatism on a finite quantile grid and, under three explicit regularity assumptions, continuous-tail conservatism on a certified interval. Our overbounding loss enforces conservativeness at selected quantiles while penalizing distributional distance with a Wasserstein-style term. The learned bounds support conservative linear-combination and convolution analysis on the enforced grid, and on the certified interval when assumptions hold, while being less redundant than traditional methods. We provide a scoped analysis of discrete-to-continuous conservatism and compact-domain objective regularity, and validate on synthetic data and real-world datasets, including multipath, ionospheric, and tropospheric residual errors. Across these settings, the method yields tighter bounds while maintaining conservatism on the enforced grid and in experiments. The framework is modality-agnostic and applicable to learning systems that require conservative, feature-conditioned uncertainty estimates in dynamic environments.
\end{abstract}


\section{Introduction}\label{sec1}
\IEEEPARstart{U}{ncertainty} quantification (UQ) is central to high-stakes applications, including autonomous systems, medical decision support, and safety monitoring. In these applications, overconfident yet wrong predictions can be catastrophic, and reliable uncertainty estimates become as important as the mean prediction itself \cite{he2025uqsurvey, dong2025survey}.
A rich set of \acrshort{UQ} methods has emerged, including Bayesian models, deep ensembles, heteroscedastic (variance) regression, \gls{QR}, and distribution-free \gls{CP} \cite{gal2016dropout,lu2021hierarchical, kendall2017uncertainties,lakshminarayanan2017simple,koenker1978regression, angelopoulos2021gentle}. These methods typically target calibration (valid coverage or well-calibrated predictive probabilities) while seeking tightness (tight intervals/distributions), a trade-off central to probabilistic forecasting \cite{gneiting2007probabilistic}.

Despite this progress, safety-critical error propagation introduces an additional requirement that is largely absent from mainstream \gls{UQ}: compositional conservatism under addition. Many systems assess risk through sums of multiple stochastic error sources arising from different modules or sensors. If the total error is a sum of components, then downstream safety metrics depend on the distribution of that sum, obtained by convolution. Unfortunately, prediction intervals produced by common \gls{UQ} pipelines are not generally closed under convolution without strong assumptions: even if each component admits a marginal $(1-\alpha)$ interval, summing these intervals does not yield a principled distributional bound for the aggregate error. A conservative workaround is Bonferroni-style splitting (e.g., $1-\alpha/n$ per component), which ensures a valid interval for the sum at level $1-\alpha$ via coarse union bounds, but can be excessively conservative and typically larger than bounds obtained from distributional convolution, especially when $n$ is large, or tails are heavy (Section~\ref{sec4_1}). Moreover, standard \gls{CP} provides finite-sample, distribution-free guarantees only for marginal coverage under exchangeability \cite{angelopoulos2021gentle,zhou2024conformalSurvey}. In contrast, conditional coverage at a given context $x$ would typically yield more adaptive (potentially tighter) intervals, but exact distribution-free conditional coverage uniformly over all $x$ is impossible without additional assumptions \cite{lei2014distributionfree,barber2021limits}. \Gls{CQR} combines \gls{QR} with \gls{CP} to better accommodate heteroscedasticity and often produces tighter intervals in practice\cite{romano2019conformalized}; nevertheless, its formal guarantee remains marginal, so the resulting intervals can still be overly conservative—especially when applied repeatedly or propagated across multiple additive components.

In contrast, Gaussian overbounding---a cornerstone of \gls{GNSS} aviation integrity monitoring---was developed precisely to enable conservative risk propagation through convolution. Gaussian overbounding replaces an empirical error distribution with a simpler distribution (often Gaussian) that is conservative in the relevant tails \cite{decleene2000defining,rife2012overbounding}. This conservatism can be transferred through summation, thereby supporting rigorous \gls{PL}(i.e., aviation-style upper bounds on error) computation. However, classical overbounding techniques (symmetric overbounding, paired overbounding, and two-step Gaussian overbounding) have two major limitations in modern data-driven settings \cite{shively2000overbound,rife2004paired,blanch2018gaussian}: (i) they often lack an explicit tightness objective, leading to excessive conservatism; and (ii) they typically assume access to the full (unconditional) error distribution, which is impractical for conditional error distributions driven by continuous context features (e.g., geometry and environment for multipath). Recent learning-based attempts mitigate the second limitation by predicting quantiles or conditional distributions using neural networks and then applying classical overbounding as a post-processing step \cite{no2021machine,liu2024overbounding,rossl2024robust}. However, this two-stage design compounds conservatism: modeling uncertainty and bounding uncertainty are optimized separately, and the post-hoc bounding step reintroduces the lack of tightness criteria.

Motivated by these gaps in safety-critical uncertainty propagation, this paper proposes a unified learning-based framework that directly trains a context-conditioned Gaussian overbound. The key idea is to embed classical overbounding constraints into the learning objective, rather than treating overbounding as a post-hoc step applied to a learned distribution or interval. Concretely, we (i) estimate multiple conditional quantiles via a multi-quantile regression objective; (ii) enforce relaxed paired-overbounding constraints (with excess mass) in quantile space to guarantee one-sided conservatism on a finite quantile grid; and (iii) introduce a Wasserstein-distance-inspired penalty to guide the solution toward a relatively tight overbound. The resulting model outputs Gaussian parameters $(\mu(x),\sigma(x))$ with finite-grid conservatism guarantees and, under three explicit regularity assumptions, continuous-tail conservatism on a certified interval suitable for convolution-based propagation, achieving a principled conservatism--tightness trade-off that bridges modern learning-based \gls{UQ} and classical integrity theory. In addition, we adopt a small ensemble across random seeds and select tail-wise conservative bounds at inference time as a heuristic safeguard. 

Our main contributions are:
\begin{itemize}
    \item This paper bridges classical Gaussian overbounding in \gls{GNSS} integrity monitoring with modern \gls{UQ} by learning context-conditioned parametric uncertainty (mean and variance) in a form that supports convolution-based risk propagation.
    \item This paper formulates conditional Gaussian overbounding as a constrained learning problem and proposes a single-stage overbounding loss that enforces conservatism in quantile space on a finite grid with an excess-mass relaxation.
    \item This paper introduces a Wasserstein-distance penalty and monotonicity regularization to explicitly trade off tightness and conservatism, yielding empirically tighter bounds than classical overbounding and prior two-stage learning-based baselines.
    \item This paper provides a scoped theoretical analysis covering discrete-to-continuous conservatism on a certified interval and local regularity of the training objective on compact domains.
    \item This paper validates the approach on synthetic mixtures and three real-world scenarios (troposphere residuals, urban multipath, ionosphere residuals), demonstrating substantial reductions in \gls{PL} while maintaining conservatism on the enforced grid and in empirical evaluation.
\end{itemize}

The remainder of this paper is structured as follows: Section~\ref{sec2} reviews background on uncertainty quantification and classical Gaussian overbounding. Section~\ref{sec3} derives the proposed overbounding loss, while the appendices provide supplementary theoretical clarification. Section~\ref{sec4} describes simulation and experimental settings. Section~\ref{sec5} presents results, and Section~\ref{sec6} concludes.

\section{Background and Related Work}\label{sec2}

This section situates the proposed framework at the intersection of modern \gls{UQ} and classical integrity overbounding. We first review representative \gls{UQ} paradigms for regression. We then summarize Gaussian overbounding methods and recent learning-based extensions.

\subsection{Uncertainty Quantification for Regression}
\Gls{UQ} for deep learning has been extensively studied, spanning Bayesian treatments, ensembles, variance networks, \gls{QR}, and distribution-free \gls{CP}; recent surveys provide comprehensive overviews of these directions \cite{he2025uqsurvey}.

\subsubsection{Bayesian methods and deep ensembles}
Bayesian deep learning represents uncertainty by placing a posterior distribution over model parameters and marginalizing predictions accordingly. Practical approximations include Monte Carlo dropout \cite{gal2016dropout} and explicit discussions of aleatoric versus epistemic uncertainty \cite{kendall2017uncertainties}. Deep ensembles train multiple independently initialized models and use predictive diversity as a scalable proxy for epistemic uncertainty, often yielding strong performance and calibration in practice \cite{lakshminarayanan2017simple,rahaman2021uncertainty}. Related probabilistic modeling tools, such as hierarchical Bayesian constructions for prediction intervals \cite{lu2021hierarchical} or Gaussian-process-style approaches \cite{gomez2023adaptive}, can also produce uncertainty estimates with appealing empirical properties. Nevertheless, these methods typically optimize average predictive performance (e.g., likelihood) and do not guarantee that extreme tails are never underestimated---precisely the failure mode targeted by integrity monitoring. Moreover, even if a full predictive distribution is produced for each individual error term, propagating uncertainty across multiple sources typically requires repeated convolutions. For arbitrary distributions, these convolutions are computationally expensive and often impractical at scale, especially in real-time or large-system settings.

\subsubsection{Variance networks and parametric predictive distributions}
A complementary line of work models uncertainty by predicting parameters of a parametric likelihood, e.g., a Gaussian with input-dependent variance (heteroscedastic regression). Such models are commonly trained using strictly proper scoring rules, most notably the negative log-likelihood \cite{gneiting2007probabilistic}. Variance networks are attractive for downstream propagation because they output closed-form distributions, and substantial effort has been devoted to making these models reliable and well-behaved in training \cite{skafte2019reliable,ma2023improving}, as well as improving distributional agreement and calibration beyond pure likelihood fitting \cite{cui2020calibrated}. However, likelihood-trained parametric models are not guaranteed to be conservative in the tails unless explicit tail constraints are imposed. Moreover, their usefulness depends on the assumed likelihood family (e.g., Gaussian), which may be misspecified for biased, skewed, or multimodal errors. This motivates learning objectives that incorporate safety constraints directly, rather than relying solely on probabilistic fit or calibration.

\subsubsection{Quantile regression}
\Gls{QR} estimates conditional quantiles by minimizing the pinball loss \cite{koenker1978regression}, providing a flexible nonparametric representation of uncertainty in quantile space. Modern deep-learning variants learn multiple quantiles jointly to form prediction bands or approximate conditional distributions \cite{tagasovska2019single}, and recent work has studied refinements beyond the pinball loss to improve calibration and robustness \cite{chung2021beyond}. Quantile formulations are appealing for safety settings because tail requirements can be expressed naturally at selected quantile levels. However, for integrity monitoring, a key limitation is that quantiles do not compose under convolution: knowing marginal conditional quantiles of each error component is insufficient to obtain tight, conservative quantiles of their sum without additional assumptions on dependence or full distributional shape.

\subsubsection{Conformal prediction and distribution-free intervals}
\Gls{CP} provides finite-sample, distribution-free guarantees for prediction sets under exchangeability \cite{shafer2008tutorial,lei2014distributionfree,angelopoulos2021gentle,zhou2024conformalSurvey}. \Gls{CQR} combines \gls{QR} with \gls{CP} to achieve marginal coverage guarantees with heteroscedasticity adaptation \cite{romano2019conformalized}, and extensions combining ensembles with \gls{CQR} have also been explored in time-series contexts \cite{jensen2022ensemble}. Recent variants further study loss-controlled or task-specific conformal mechanisms for regulating uncertainty \cite{wang2024conformal,wang2025satellite}. Despite their rigor for single-target uncertainty, \gls{CP} guarantees are typically marginal rather than conditional, and achieving nontrivial distribution-free conditional coverage is fundamentally limited without additional assumptions \cite{barber2021limits}. In high-reliability regimes (very small miscoverage), \gls{CP}-style intervals can become overly wide, reducing availability. Additionally, \gls{CP} are not designed for convolution-transitive propagation across multiple error sources; a common workaround is to allocate miscoverage as $\alpha/n$ per source and apply a union bound, yielding valid but often overly conservative bounds for sums (see Section~\ref{sec4_1}).

\subsection{Gaussian Overbounding and Uncertainty Propagation}

In safety-critical navigation, the objective is not merely calibrated intervals for a single prediction, but conservative distributional bounds that remain conservative under the addition of multiple independent error sources. Gaussian overbounding addresses this by replacing an empirical error distribution with an analytically tractable Gaussian distribution that conservatively dominates the relevant tails. The overbounding condition is commonly stated in terms of \gls{CDF} overbounding \cite{decleene2000defining, rife2007symmetric}:
\begin{align}
    F_a(x) &\geq F_{oa}(x), \qquad \forall F_a(x)\geq \frac{1}{2} \label{eq1}\\
    F_a(x) &\leq F_{oa}(x), \qquad \forall F_a(x)\leq \frac{1}{2} \label{eq2}
\end{align}
where $F_a(x)$ is the \gls{CDF} of the actual distribution for an arbitrary variable $a$, and $F_{oa}(x)$ denotes the \gls{CDF} of its corresponding overbounding distribution.

A key requirement is transitive conservativeness: if each component is overbounded, then the sum should also be overbounded after convolution. Concretely, if $a$ is overbounded by $oa$, then for an independent error $b$ the distribution of $(oa+b)$ should conservatively bound that of $(a+b)$:
\begin{align}
    F_{a+b}(x) &\geq F_{oa+b}(x), \quad\forall F_{a+b}(x)\geq \frac{1}{2} \label{eq3}\\
    F_{a+b}(x) &\leq F_{oa+b}(x), \quad\forall F_{a+b}(x)\leq \frac{1}{2} \label{eq4}
\end{align}
And the \gls{CDF} of sum of $(a+b)$ can be expressed as:
\begin{equation}
    F_{a+b}(x) = \int_{-\infty}^x (f_a * f_b)(t) dt \label{eq5}
\end{equation}
where $f_a * f_b$ represents the convolution of \gls{PDF} $f_a$ and $f_b$. This closure under convolution enables rigorous computation of \gls{PL} in integrity monitoring.

Based on this overbounding requirement, Rife \cite{rife2007symmetric} finds that if $a$ and $b$, and $oa$ are symmetric and unimodal distributions, the conservativeness transitivity can be maintained, a technique known as symmetric overbounding. However, symmetric assumptions are often unrealistic for biased or multimodal distributions like urban multipath errors \cite{liu2024overbounding}.

Paired overbounding, introduced by Rife \cite{rife2004paired}, relaxes symmetry constraints by bounding the actual distribution with two Gaussian distributions—left and right bounds—expressed as:
\begin{align}
F_a(x) &\le F_L(x), \qquad \forall\, x \label{eq6}\\
F_a(x) &\ge F_R(x), \qquad \forall\, x \label{eq7}
\end{align}
where $F_{L}$ and $F_{R}$ are the \gls{CDF} left and right bounds of the actual distribution. The overbounding \gls{CDF} is then constructed from these left and right bounds:
\begin{equation}
\label{eq8}
    F_{oa}=\begin{cases}
    F_{L}(x) &\text{if}\quad F_{L}(x)< {\frac{1}{2}}\\
    {\frac{1}{2}}&{\rm otherwise}\\
    F_{R}({x}) &\text{if}\quad F_{R}(x)>{\frac{1}{2}}
    \end{cases}
\end{equation}
However, the paired overbounding requires the actual distribution, and \eqref{eq6} and \eqref{eq7} are stringent for finite real-world data distributions. Take \eqref{eq6} for example, because the real error distribution based on real-world data is finite, which can have $F_{a}(x_0) = 1, x_0<\infty$; while for any Gaussian distribution, $F_{L}(x_0)<1$.

Blanch \cite{blanch2018gaussian} proposed two-step Gaussian overbounding, combining symmetric and paired overbounding by two-step and using an excess mass parameter ($\epsilon$) to relax stringent criteria. Initially, a symmetric unimodal distribution $F_{su}(x)$ overbounds the actual distribution using paired methods relaxed by $\epsilon$. The excess mass parameter relaxes \eqref{eq6} and \eqref{eq7} to solve the finite-sample limitation. Subsequently, a Gaussian distribution is employed to overbound the symmetric unimodal model using the symmetric overbounding technique. The procedure can be detailed as follows, using the left-side overbound as an illustrative example:
\begin{align}
    F_{a}(x) \leq (1+\epsilon)F_{su}(x), \quad \forall x \label{eq9} \\
    F_{su}(x) \leq F_{L}(x), \quad \forall F_{L}(x)\leq \frac{1}{2} \label{eq10}
\end{align}

For multiple $n$ additive errors, this ensures:
\begin{equation}
\begin{aligned}
\label{eq11}
    & F_{a_1+a_2...+a_n}(x)\leq \prod\limits_{i=1}^{n}(1+\epsilon_i) F_{su_1+su_2...+su_n}(x) \\
    & \leq \prod\limits_{i=1}^{n}(1+\epsilon_i) F_{L_1+L_2...+L_n}(x), \quad \forall F_{L_1+L_2...+L_n}(x)\leq \frac{1}{2}
\end{aligned}
\end{equation}
Additional research efforts have explored relaxing the strict Gaussian assumption by using hybrid models—employing a Gaussian core for the main body of the distribution and alternative forms for the tail behavior—to reduce excessive conservatism \cite{larson2019improving}.  Despite these refinements, these methods often (i) lack explicit tightness objectives, relying on search heuristics that may yield overly conservative bounds, and (ii) require full knowledge of the actual error distribution, limiting applicability to conditional error modeling with continuous context features.

\subsection{Learning-Based Overbounding and Two-Stage Pipelines}
Recent works incorporate neural networks (\gls{MLP} and \gls{LSTM}) \cite{no2021machine, liu2024overbounding} to predict conditional distributions via estimated quantiles. Subsequently, a single predicted quantile is used to define a zero-mean Gaussian distribution that matches this quantile—an approach called quantile overbounding. However, assuming a zero-mean Gaussian is a strong constraint that may either result in excessive conservatism or fail to satisfy conservatism, requiring a posteriori verification. Another work \cite{rossl2024robust} has employed \gls{MLP} to predict multiple quantiles at different levels to reconstruct the conditional distribution, followed by applying two-step Gaussian overbounding. These methods can be categorized as two-stage frameworks: the first stage involves learning the conditional distribution using quantile estimates, and the second applies Gaussian overbounding. Although this reduces reliance on explicit distribution knowledge, it inherits two drawbacks: (i) the two-stage structure accumulating conservatism across both stages; and (ii) the absence of built-in mechanisms to minimize the redundancy.

In summary, existing \gls{UQ} approaches (e.g., \gls{QR}/\gls{CP}/variance regression) provide calibrated uncertainty for single-target prediction, yet they do not directly yield the convolution-transitive conservative bounds required for integrity risk propagation across multiple sources. Classical Gaussian overbounding enables such propagation but is difficult to extend to feature-conditioned error distributions and often lacks an explicit tightness objective. In the next section, we present a single-stage learning formulation that integrates overbounding constraints with an explicit tightness criterion to obtain context-conditioned Gaussian bounds suitable for convolution-based conservativeness transitivity.

\section{Overbounding Loss Derivation}\label{sec3}
This section proposes a single-stage overbounding loss enforced on a finite quantile grid.
By optimizing this loss, a neural network directly outputs the mean and standard deviation of a context-conditioned Gaussian overbound.
From a \gls{UQ} perspective, our goal is to achieve (i) one-sided conservatism in the integrity sense on the enforced grid so that the bound supports convolution-based propagation, and (ii) tightness---to avoid unnecessary standard deviation inflation and preserve system availability. Appendix~\ref{app3} clarifies how the finite-grid constraint extends to a certified continuous interval, while Appendix~\ref{app4} summarizes why the objective remains regular on compact domains.

\subsection{Motivation and Formulation}\label{sec3_1}

To address the limitations identified in Section \ref{sec2}, we propose a unified neural network-based method that merges the two-stage overbounding process into a single-stage framework, thus minimizing computational redundancy and reducing excessive conservatism. Specifically, we utilize multiple quantiles to represent the conditional distribution and integrate multiple quantile regression with relaxed paired overbounding constraints to form an "overbounding loss." Additionally, we incorporate the Wasserstein distance into this loss to guide the model toward a near-minimal overbounding distribution under constraints, effectively balancing conservatism and tightness.

Given a dataset $\mathcal{D}=\{x_i,y_i\}^N_{i=1}$, where each $y_i$ is sampled from a conditional distribution $Z|x_i$, our objective is to estimate the conditional quantile $q_\tau$ corresponding to a given quantile level $\tau\in\mathcal{T}$, $\mathcal{T} \subseteq [0, 1]$. The conditional \gls{CDF} of $Z|x$ is denoted as $F_{y\sim Z|x}(y)$, thus $\tau = F_{y\sim Z|x}(q_\tau)$.

We begin by reviewing quantile regression. Estimating a quantile of a distribution involves optimizing a parameterized function $\theta_\tau(x)$ to minimize the quantile regression loss $\mathcal{J}_\tau^{QL}$, defined as:
\begin{align}
&\hat{q}_\tau = \hat{F}^{-1}(\tau)= \hat{\theta}_\tau(x) = \mathop{\arg\min}\limits_{\theta}(\mathcal{J}^{QL}_\tau (y,\theta_\tau(x))) \label{eq12}\\
&\mathcal{J}_\tau^{QL}(y,\theta_\tau(x)) := \mathbb{E}[\mathcal{L}^{QL}_\tau (y-\theta_\tau(x))] \label{eq13}\\
&\mathcal{L}_\tau^{QL}(\epsilon) = \epsilon(\tau- \mathbb{I}(\epsilon<0)), \quad \forall \epsilon \in \mathbb{R} \label{eq14}
\end{align}
where $\mathbb{I}(\cdot)$ is the indicator function.

\subsection{Integrating Paired Overbounding with Excess Mass}\label{sec3_2}
Combining paired overbounding in \eqref{eq6}-\eqref{eq7} and excess mass techniques in \eqref{eq9} in quantile space, we aim to identify Gaussian distributions $F_L$ and $F_R$, parameterized by mean and standard deviation pairs ($\mu_L$, $\sigma_L$) and ($\mu_R$, $\sigma_R$), satisfying:
\begin{align}
    \tau = F(q_\tau)\leq (1+\epsilon) F_{L}(q_\tau), \quad \forall \tau\in\mathcal{T} \label{eq15}\\
1-\tau \leq (1+\epsilon)(1-F_{R}(q_\tau)), \quad \forall \tau\in\mathcal{T} \label{eq16}
\end{align}

These constraints can be restated explicitly in terms of quantile levels as:
\begin{align}
&F_{L}(F^{-1}(\tau))\geq \frac{\tau}{1+\epsilon}, \quad \forall \tau\in\mathcal{T} \label{eq17}\\
&F_R(F^{-1}(\tau)) \leq \frac{\tau+\epsilon}{1+\epsilon}, \quad \forall \tau\in\mathcal{T} \label{eq18}
\end{align}
In the training objective below, these conservatism constraints are enforced on the finite grid $\mathcal{T}$. Appendix~\ref{app3} introduces a certified interval and a quantile-gap function showing how the discrete constraints extend to a continuous interval under additional regularity assumptions.

Thus, the final overbounding distribution $F_{oa}$ is constructed as:
\begin{equation}
F_{oa}(x)=\begin{cases}
(1+\epsilon)F_{L}(x), &\text{if }(1+\epsilon)F_{L}(x)< \frac{1}{2}\\[6pt]
\frac{1}{2}, &\text{otherwise}\\[6pt]
(1+\epsilon)F_R(x)-\epsilon, &\text{if }(1+\epsilon)F_R(x)-\epsilon>\frac{1}{2}
\end{cases}\label{eq19}
\end{equation}

Therefore, the overbounding problem is transformed into constructing the final left and right bounding distributions that satisfy \eqref{eq17} and \eqref{eq18} across all quantile levels $\tau \in \mathcal{T}$. Here, we consider the left bound $F_L$ as an illustrative example. For each quantile level $\tau$, the objective is to determine a Gaussian distribution, parameterized by mean $\mu_\tau$ and standard deviation $\sigma_\tau$, whose \gls{CDF} conservatively bounds the actual distribution, satisfying \eqref{eq17}. Specifically, we define:
\begin{equation}
F_{L_\tau}(q_\tau)=k_\tau \frac{\tau}{1+\epsilon}, \quad \frac{1+\epsilon}{\tau}\geq k_\tau\geq 1 \label{eq20}
\end{equation}

Here, $k_\tau$ is a local overbounding factor at quantile level $\tau$ that ensures the conservative \gls{CDF} condition is satisfied for the auxiliary bound $F_{L_\tau}$. It also quantifies how closely the overbound distribution aligns with the target quantile. Consequently, the quantile $q_\tau$ can be expressed explicitly as:
\begin{equation}
\begin{aligned}\label{eq21}
q_\tau &= F_{L_\tau}^{-1}\left(k_\tau\frac{\tau}{1+\epsilon}\right) \\
&= \mu_\tau + Q^{-1}\left(k_\tau\frac{\tau}{1+\epsilon}\right)\sigma_\tau,
\end{aligned}
\end{equation}
where $Q^{-1}(\cdot)$ denotes the inverse \gls{CDF} (quantile function) of the standard Gaussian distribution.

By integrating \eqref{eq21} into quantile regression, the estimated quantile $\hat{q}_\tau$ is obtained by optimizing the parameter set $\theta_\tau(x)=\{\mu_\tau(x),\sigma_\tau(x),k_\tau(x)\}$ through minimizing the quantile regression loss:
\begin{equation} \label{eq22}
\begin{aligned}
\hat{q}_\tau &= \hat{\theta}_\tau(x) = \mathop{\arg\min}_{\theta_\tau} \mathcal{J}^{QL}_\tau(y,\theta_\tau) \\
&= \mathop{\arg\min}_{\mu_\tau, \sigma_\tau, k_\tau} \mathcal{J}^{QL}_\tau\left(y,\, \mu_\tau(x) + \sigma_\tau(x) Q^{-1}\!\left(k_\tau(x)\tfrac{\tau}{1+\epsilon}\right)\right)
\end{aligned}
\end{equation}
with constraints:
\begin{equation} \label{eq23}
\frac{1+\epsilon}{\tau}\geq k_\tau\geq 1.
\end{equation}

As the magnitude of each quantile loss $\mathcal{J}^{QL}_\tau$ inherently depends on the quantile level $\tau$, normalization is necessary to balance contributions across different quantiles. This normalization also highlights the distribution tails, enhancing conservatism in the predictions. Consequently, we introduce reweighting in the loss function:
\begin{equation}
\mathcal{J}^{MQL} = \sum_{\tau \in \mathcal{T}} w_\tau \mathcal{J}^{QL}_\tau,
\label{eq26}
\end{equation}
where the weight  is defined as:
\begin{equation}
w_\tau = \frac{1}{4\tau(1-\tau)}.
\label{eq27}
\end{equation}

A detailed derivation of this weighting factor is provided in Appendix A. This approach ensures each quantile contributes roughly equally, emphasizing tail quantiles to maintain robust and conservative predictions.

Therefore, even after obtaining a set of unique quantile solutions $\{\hat{q}_\tau\}_{\tau\in\mathcal{T}}$, multiple combinations of parameters $(\mu_\tau,\sigma_\tau,k_\tau)$ can still satisfy \eqref{eq21}. It is thus essential to control the degrees of freedom of these parameters through additional constraints and penalty terms. The final overbounding distribution must fulfill the overbounding criteria across all quantile levels and be optimal in terms of conservativeness and tightness. Next, we analyze various parameter constraints and introduce penalty terms that yield a relatively optimal and conservative overbounding distribution.

The final overbounding distribution ${F}_{L}(\hat{q}_\tau)$, parameterized by a selected pair $(\mu, \sigma)$ from $\{(\mu_\tau, \sigma_\tau)\}_{\tau\in \mathcal{T}}$, must satisfy the overbounding condition in \eqref{eq15}. A sufficient condition to guarantee conservativeness is that its \gls{CDF} dominates all auxiliary bounds:
\begin{equation}\label{eq28}
    F_{L}(\hat{q}_\tau)
    =F_{\mathcal{N}(\mu, \sigma)}(\hat{q}_\tau)
    \;\geq\;
    F_{L_\tau}(\hat{q}_\tau)
    =F_{\mathcal{N}(\mu_\tau,\sigma_\tau)}(\hat{q}_\tau),
    \  \forall \tau \in \mathcal{T},
\end{equation}
where $F_{\mathcal{N}(\mu, \sigma)}$ denotes the \gls{CDF} of a Gaussian distribution with mean $\mu$ and standard deviation $\sigma$.

To make this relation explicit in quantile space, we introduce a global overbounding factor $K_\tau$ at each quantile level $\tau$ and write
\begin{align}
    F_{L}(\hat{q}_\tau)
    &= K_\tau \frac{\tau}{1+\epsilon}
    \;\geq\;
    F_{L_\tau}(\hat{q}_\tau)
    = \frac{k_\tau \tau}{1+\epsilon},
    \label{eq29}\\
    \implies\quad
    K_\tau &\geq k_\tau,
    \label{eq30}
\end{align}
where $k_\tau$ is the local overbounding factor defining the auxiliary bound $F_{L_\tau}$ in \eqref{eq20}, and $K_\tau$ is the corresponding global overbounding factor induced by the final bound $F_L$ at quantile level $\tau$. By construction, \eqref{eq30} guarantees that $F_L$ is at least as conservative as $F_{L_\tau}$ for every $\tau$.

Applying the inverse \gls{CDF} to both $F_L$ and $F_{L_\tau}$ then yields
\begin{equation}
\begin{aligned}\label{eq31}
    \hat{q}_\tau
    &= \mu_\tau
       + Q^{-1}\!\left(k_\tau\frac{\tau}{1+\epsilon}\right)\sigma_\tau \\
    &= \mu
       + Q^{-1}\!\left(K_\tau\frac{\tau}{1+\epsilon}\right)\sigma
       \;\geq\;
       \mu + Q^{-1}\!\left(k_\tau\frac{\tau}{1+\epsilon}\right)\sigma,
\end{aligned}
\end{equation}
To stabilize optimization and reduce redundancy in the parameter space, we next consider three constraint scenarios on $(\mu_\tau,\sigma_\tau,k_\tau)$ and their shared counterparts.

\subsection{Case 1: Constant $\mu_\tau$ and $k_\tau$ across quantile levels}\label{sec3_3}

If we enforce $\mu_\tau = \mu, k_\tau =k$, for all $\tau \in \mathcal{T}$, optimization parameters reduce to $(\mu, \{\sigma_\tau\}_{ \tau\in \mathcal{T}}, k)$. Given \eqref{eq31}, we derive the necessary condition:
\begin{equation}
    Q^{-1}(k\frac{\tau}{1+\epsilon}) \sigma_\tau \geq Q^{-1}(k\frac{\tau}{1+\epsilon})\sigma \text{, }\forall\tau\in \mathcal{T} \label{eq32}
\end{equation} 
Since the sign of $Q^{-1}(k\frac{\tau}{1+\epsilon})$ depends on $\tau$, the inequality results in two opposing constraints for $\sigma$:
\begin{equation}\label{eq33}
\sigma = \begin{cases}
    \sigma\geq\sigma_{up} = \max(\sigma_\tau),  &\text{if }\tau\leq \frac{1+\epsilon}{2k}\\
    \sigma \leq \sigma_{lo} = \min(\sigma_\tau), &\text{if } \tau \geq \frac{1+\epsilon}{2k}
\end{cases}
\end{equation}

In general, this leads to contradictions ($\sigma_{up}\leq \sigma\leq \sigma_{lo}$ cannot always hold). Therefore, enforcing constant $\mu_\tau$ and $k_\tau$ across all quantile levels is typically infeasible.

\subsection{Case 2: Constant $\sigma_\tau$ and $k_\tau$ across quantile levels}\label{sec3_4}
If we set $\sigma_\tau = \sigma, k_\tau = k$ for all $\tau$, the optimization variable become $(\{\mu_\tau\}_{\tau \in\mathcal{T}}, \sigma, k)$. According to \eqref{eq31}, this implies $\mu_\tau \geq \mu, \quad \forall \tau \in \mathcal{T}$, thus $\mu = \min(\mu_\tau)$. To guarantee positivity for $\sigma$, we optimize its logarithmic transformation: $\log_{\sigma} = \log(\sigma)$. Moreover, the parameter $k_\tau$ is constrained to the interval $(1,\frac{1+\epsilon}{\tau})$. Therefore, the parameter $k$ must satisfy $1 \leq k\leq \min(\frac{1+\epsilon}{\tau}) \to 1\leq k \leq1+\epsilon$, which we enforce via a sigmoid transformation:
\begin{equation}\label{eq34}
    k = 1 + \epsilon \cdot \text{sigmoid}(s_k)
\end{equation}
with unconstrained parameter $s_k$. Thus, the optimization parameters become $(\log_{\sigma}, \{\mu_\tau\}_{\tau \in\mathcal{T}}, s_k)$. 
Still, multiple valid solutions may exist: different auxiliary parameter combinations can produce the same quantile values through \eqref{eq21}, so the representation is not unique. Penalty terms (see Section \ref{sec3_6}) are therefore included to encourage stable optimization behavior toward a useful conservative solution.

\subsection{Case 3: Constant $\mu_\tau$ and $\sigma_\tau$ across quantile levels}\label{sec3_5}
Enforcing $\mu_\tau=\mu$ and $\sigma_\tau=\sigma$ for all $\tau$ ensures \eqref{eq31} is always satisfied, making this scenario feasible. To ensure $\sigma > 0$, we optimize its logarithm $\log_{\sigma}$. Each $k_\tau$ must lie within $(1,\frac{1+\epsilon}{\tau})$, enforced through individual sigmoid transformations:
\begin{equation} \label{eq35}
    k_\tau = 1 + (\frac{1+\epsilon}{\tau}-1) * \text{sigmoid}(s_{k_\tau})
\end{equation}
where $s_{k_\tau}$ is the unconstrained parameter for each quantile level. Thus, optimization variables reduce to $(\mu, \log_{\sigma}, \{s_{k_\tau}\}_{\tau \in\mathcal{T}})$. 
As in Case~2, the free auxiliary variables are not unique, so parameter sharing and penalties are used to select a stable conservative representative rather than to recover a unique latent triplet for every quantile.

\subsection{Penalty Term Formulation}\label{sec3_6}
To evaluate the quality of an overbounding distribution, two primary aspects should be considered: (1) adherence to conservative constraints and (2) minimality of the tightness or induced \gls{PL} (in aviation). However, directly evaluating the \gls{PL} in aviation is challenging because it depends on multiple factors such as integrity risk (a safety requirement), the number and characteristics of error sources, and the specific computation methods, making it difficult to derive a fixed analytical expression for all systems.

To address this challenge in a more tractable manner, we propose measuring the distance between the \gls{CDF} of the actual distribution and the corresponding overbounding distribution. Because directly measuring distances in probability space is often impractical due to unknown actual probabilities for the conditional distribution, we instead utilize the Wasserstein distance, also known as the \gls{EMD}, which quantifies differences between distributions effectively in quantile space. A smaller Wasserstein distance indicates that the overbounding distribution closely approximates the actual distribution across quantile levels.

Formally, the general p-th order Wasserstein distance ($W_p$) is defined as:
\begin{equation}\label{eq36}
W_p = \left(\int_0^1\left|F_{oa}^{-1}(\tau)-F_{a}^{-1}(\tau)\right|^p d\tau\right)^{\frac{1}{p}}
\end{equation}
From the monotonicity property of the \gls{CDF} and the overbounding condition in \eqref{eq17}, we derive:
\begin{align}
F^{-1}(\tau)\geq F_L^{-1}\left(\frac{\tau}{1+\epsilon}\right), \quad \forall \tau\in\mathcal{T}. \label{eq37}
\end{align}

This inequality indicates that in quantile space, $F_L^{-1}(\frac{\tau}{1+\epsilon})$ must conservatively overbound $F^{-1}(\tau)$ for all enforced quantile levels, and by Appendix~\ref{app3}, Theorem~\ref{thm:grid_to_continuum}, also on the entire certified interval when the stated regularity assumptions hold.
For the left tail, we have $F_{oa}(x)=(1+\epsilon)F_L(x)$ whenever $(1+\epsilon)F_L(x)<\tfrac12$. Equivalently, in quantile space this corresponds to $\tau\in[0,\tfrac12]$, under which
$F_{oa}^{-1}(\tau)=F_L^{-1}\!\bigl(\tfrac{\tau}{1+\epsilon}\bigr)$.
Therefore, the left-tail Wasserstein distance $W_{p_L}$ can be written as:

\begin{align}
W_{p_L}
&:= \left(\int_{0}^{\tfrac12}
\left|F_{oa}^{-1}(\tau)-F^{-1}(\tau)\right|^{p}\,d\tau\right)^{\tfrac{1}{p}}
\nonumber \\
&= \left(\int_{0}^{\tfrac12}
\left|F_L^{-1}\!\left(\tfrac{\tau}{1+\epsilon}\right)-F^{-1}(\tau)\right|^{p}\,d\tau\right)^{\tfrac{1}{p}}
\label{eq39}
\end{align}

For $p=1$, the left side of Wasserstein-1 distance $W_{1_L}$ for overbounding task becomes:

\begin{equation}\label{eq40}
\begin{aligned}
W_{1_L}
={}& \int_0^{\frac{1}{2}}
\left|F^{-1}(\tau)-F_L^{-1}\!\left(\frac{\tau}{1+\epsilon}\right)\right|\,d\tau \\
\propto{}& \sum_{\tau\in\mathcal{T},\,\tau<\frac{1}{2}}
\left|q_\tau-\left(\mu+Q^{-1}\!\left(\frac{\tau}{1+\epsilon}\right)\sigma\right)\right| \\
\approx{}& \sum_{\tau\in\mathcal{T},\,\tau<\frac{1}{2}}
\Bigl|\mu_\tau + Q^{-1}\!\left(k_\tau\frac{\tau}{1+\epsilon}\right)\sigma_\tau \\
&\qquad\ -\left(\mu+Q^{-1}\!\left(\frac{\tau}{1+\epsilon}\right)\sigma\right)\Bigr|
\end{aligned}
\end{equation}

To encourage the optimization to favor a tight yet conservative overbounding distribution, we incorporate the Wasserstein distance as a penalty term into the loss function. This penalty term, $\mathcal{J}_{p}$, is explicitly defined as:

\begin{equation}\label{eq41}
\begin{aligned}
\mathcal{J}_{p} ={}& \sum_{\tau \in \mathcal{T}, \tau < \frac{1}{2}} 
\left| \mu_\tau + Q^{-1}\left(k_\tau\frac{\tau}{1+\epsilon}\right)\sigma_\tau \right. \\
&\qquad \left. - \left(\mu + Q^{-1}\left(\frac{\tau}{1+\epsilon}\right)\sigma\right)\right|
\end{aligned}
\end{equation}
Thus, the complete cost function combines the multiple quantile loss ($\mathcal{J}^{MQL}$) and this penalty term, balanced by a regularization parameter  $\lambda$:
\begin{equation}
    \mathcal{J} = \mathcal{J}^{MQL} + \lambda \mathcal{J}_{p}
\end{equation}

Introducing the penalty term $\mathcal{J}_{p}$ modifies the quantile estimates slightly, potentially violating the conservatism condition ($\hat{q}_\tau \leq q_\tau$). Specifically, if $F(\hat{q}_\tau)>\tau$, the conservatism constraint is breached. To guard against penalty-induced underestimation, we scale each quantile level $\tau$ by a factor $t$ in the quantile loss terms, replacing $\tau$ with $t\tau$.
For parameter configurations analyzed (Cases 2 and 3 from the previous section), Appendix~\ref{app2} derives sufficient first-order stationarity conditions for the scaling factor $t$. The analysis concludes that Case 3 is numerically unreliable in the extreme tail and is therefore not adopted as the main method. In contrast, Case 2 provides the following worst-case sufficient condition (corresponding to $n_{\min}=1$):
\begin{equation}
    t \leq 1-(n_L-1) \frac{\lambda}{w_\tau \tau}
\end{equation}
where $n_L$ denotes the number of quantiles satisfying $\tau<\frac{1}{2}$. Since this derivation is grid-based, we interpret the practical choice of
\begin{equation}
    t=1-200\lambda
\end{equation}
as a heuristic finite-grid safety margin.

To ensure monotonicity among the estimated quantiles, we introduce a monotonicity penalty term, denoted as  $\mathcal{J}_m$. This term ensures the sequential consistency of quantile estimates across increasing quantile levels, thereby enhancing model stability. The monotonicity penalty is formulated as follows:
\begin{equation}
    \mathcal{J}_m = \sum_{i=1}^{T-1} \max\left(\hat{q}_{\tau_i}-\hat{q}_{\tau_{i+1}}, 0\right),
\end{equation}
where $T$ represents the total number of quantile levels and $\tau_1<\tau_2<\cdots<\tau_T$.
Consequently, the final unified cost function incorporating the multiple quantile loss, Wasserstein distance-based penalty, and quantile monotonicity constraints is expressed as:

\begin{equation}\label{eq:full_loss}
\begin{aligned}
\mathcal{J} ={ }&
\sum_{\tau \in \mathcal{T}} w_\tau \frac{1}{n} \sum_{i=1}^{n}
\left( y_i - \left( \mu_\tau + Q^{-1}\!\left( k_\tau \frac{\tau}{1+\epsilon} \right) \sigma_\tau \right) \right) \\
& \times \left( t\,\tau - \mathbb{I}\!\left( y_i - \left( \mu_\tau + Q^{-1}\!\left( k_\tau \frac{\tau}{1+\epsilon} \right) \sigma_\tau \right) < 0 \right) \right) \\
& + \lambda \sum_{\substack{\tau \in \mathcal{T}, \\ \tau < \frac{1}{2}}}
\left| \mu_\tau + Q^{-1}\!\left( k_\tau \frac{\tau}{1+\epsilon} \right) \sigma_\tau \right. \\
& \quad \quad \quad \left. - \left( \mu + Q^{-1}\!\left( \frac{\tau}{1+\epsilon} \right) \sigma \right) \right| \\
& + \beta \sum_{i=1}^{T-1} \max\!\left( \hat{q}_{\tau_i} - \hat{q}_{\tau_{i+1}},\, 0 \right)
\end{aligned}
\end{equation}

where the hyperparameters $\lambda$ and $\beta$ balance the importance of the Wasserstein penalty and the monotonicity constraint, respectively.
Appendix~\ref{app4} establishes that, on compact parameter domains under Case~2, the objective in \eqref{eq:full_loss} is locally Lipschitz, differentiable almost everywhere, and admits Clarke subgradients. This supports the first-order stationarity analysis in Appendix~\ref{app2}.

Regarding the right bound, the identical procedure applies symmetrically. Specifically, we derive the left bound $\bar{F_L}(x)$ of the conditional distribution $-y_i\sim Z|x_i$. The corresponding right bound distribution $F_R$ thus shares the standard deviation but possesses an opposite mean compared to $\bar{F_L}(x)$, completing the overbounding construction.

\section{Simulation and Experimental Settings}\label{sec4}
This section describes the configurations used for both the simulation and real-world experiments conducted to validate the proposed overbounding method.
\subsection{Simulation with Different Distribution Types}\label{sec4_1}
To evaluate the general applicability of our method, we designed simulations involving three distinct synthetic distributions. For each type, 300,000 samples were drawn from Gaussian mixture models with three components, specified as follows:
\begin{itemize}
\item Type 1 (Multimodal): A mixture of Gaussian distributions $\mathcal{N}(-5,1)$, $\mathcal{N}(0,2)$, and $\mathcal{N}(5,4)$, each with equal weight. This distribution is approximately zero-mean, exhibiting multimodality and left-side concentration.
\item Type 2 (Negative Mean): A mixture of $\mathcal{N}(-5,1)$, $\mathcal{N}(0,2)$, and $\mathcal{N}(0,4)$, each equally weighted. This distribution is skewed negatively, displaying left-heavy tails.
\item Type 3 (Positive Mean, Heavy Tail): A mixture of $\mathcal{N}(0,1)$, $\mathcal{N}(0,2)$, and $\mathcal{N}(5,4)$, equally weighted. This distribution has a positive mean with heavy right-side tails.
\end{itemize}

Fig. \ref{fig:simulation_distributions} illustrates the \gls{PDF} of these distributions along with their respective histograms.
\begin{figure*}[!htbp]
    \centering
    \subfloat[]{\includegraphics[width=0.3\textwidth]{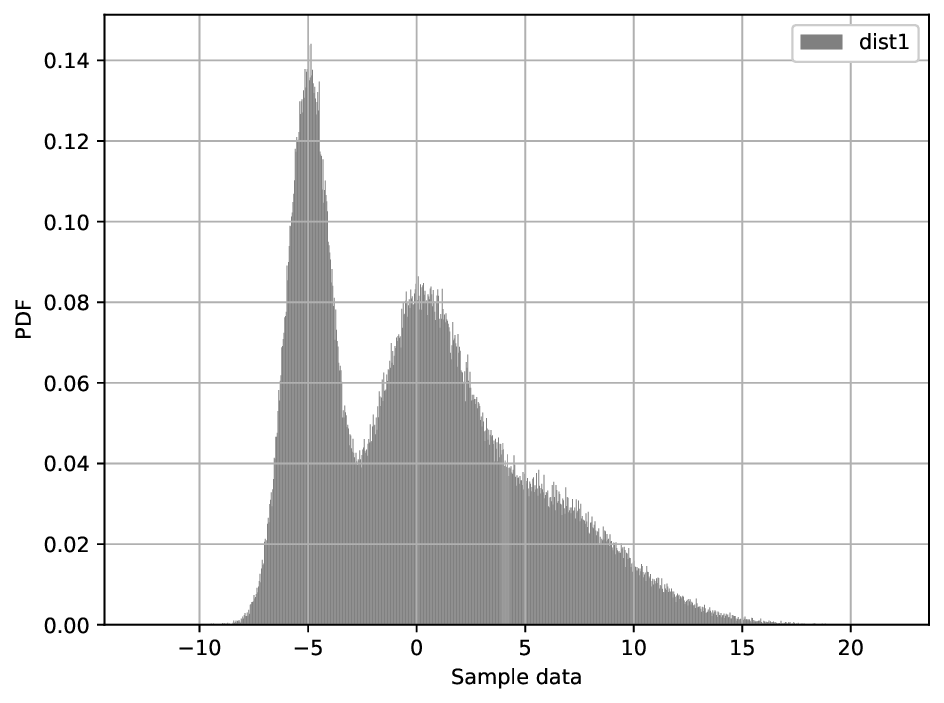}}
    \hfill
    \subfloat[]{\includegraphics[width=0.3\textwidth]{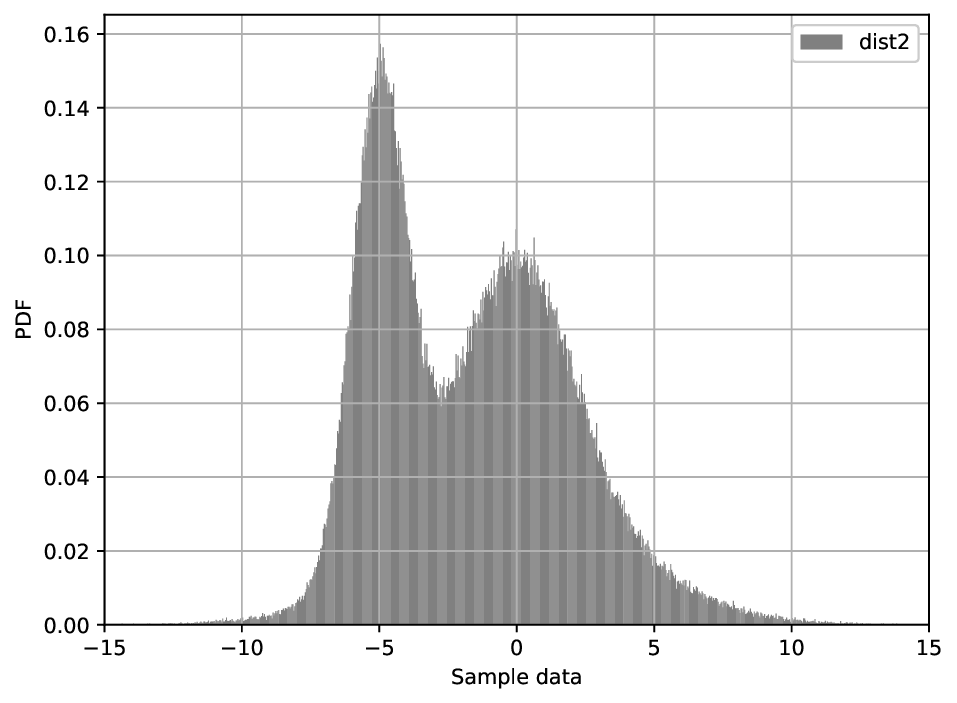}}
    \hfill
    \subfloat[]{\includegraphics[width=0.3\textwidth]{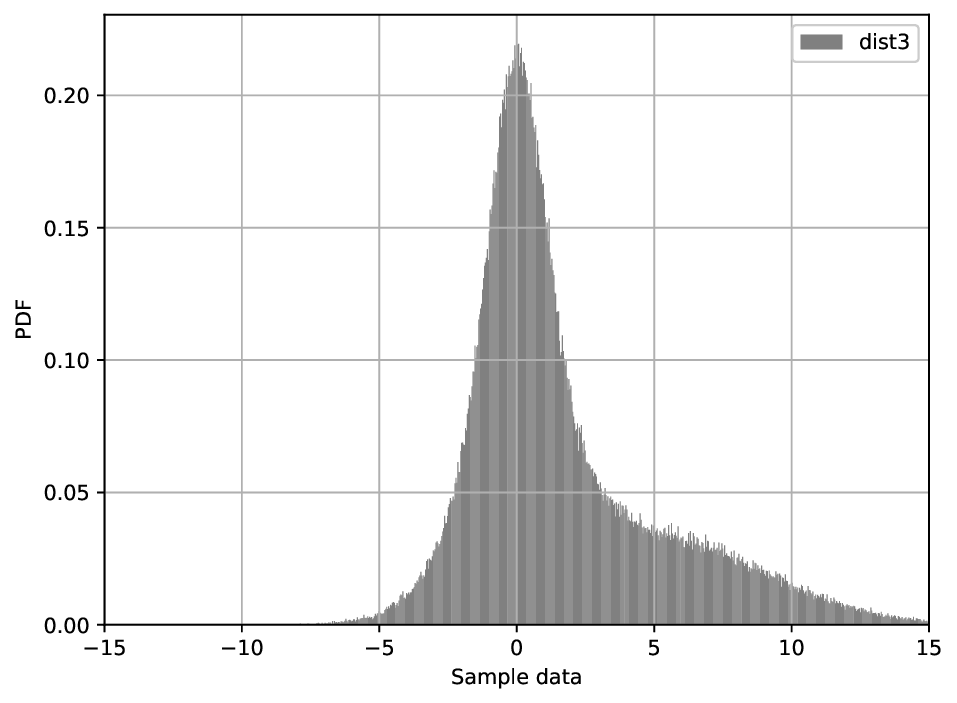}}
    \caption{Histograms of the three representative distribution types:
    (a) Type 1 distribution, (b) Type 2 distribution, and (c) Type 3 distribution.}
    \label{fig:simulation_distributions}
\end{figure*}

Because the proposed loss is backbone-agnostic, it can be combined with various architectures (e.g., \gls{MLP}, \gls{LSTM}). To isolate the behavior of the loss, in synthetic simulations we directly optimized $\{\mu_\tau,s_k,\log_\sigma\}$ as outlined in Section~\ref{sec3}, without a neural network backbone. To mitigate potential gradient instability, $s_k$ was clamped within $[-15,15]$.
Table~\ref{tab:training_hyperparameters} summarizes the specific training hyperparameters, consistent across all simulations. The excess mass parameter $\epsilon$ is set according to \cite{blanch2018gaussian}.

\begin{table}[!htbp]\begin{center}
\caption{Training hyperparameters \label{tab:training_hyperparameters}}
\centering
\begin{tabular}{l|l}
    \hline
    Setting                    & Value \\ \hline
    Epochs                     & 50000     \\ 
    Optimizer                  & Adam    \\
    Batch size                 & 10000       \\ 
    Start learning rate (LR)   & 0.01     \\
    Learning rate schedule     & Cosine   \\ 
    Warmup epochs              & 20       \\
    Weight decay               & 0.0      \\ 
    Cooldown epochs            & 30       \\ 
   $\lambda$                   & 0.00001 \\ 
   $\beta$                     & 0.001 \\ 
    $\epsilon$                 & 0.0025 \\ \hline
\end{tabular}\end{center}
\end{table}

We compare our method (specifically, Case 2 as defined in Section \ref{sec3}) against the two-step Gaussian overbounding (Two-Step OB), paired overbounding (Paired OB) with excess mass technique \cite{blanch2018gaussian}, and quantile overbounding (Quantile OB), with quantile level at 0.99. Performance metrics used for comparison include \gls{PL}, Wasserstein distance, and average global overbounding factor $\mathcal{K}_\tau$.

The protection level calculation follows the approach described by \cite{blanch2018gaussian}. Considering $n$ independent random variables $x_k$, each overbounded by of $\mathcal{N}(\mu,\sigma)$, their mean $y = \frac{1}{n}\sum_{k=1}^n x_k$ is conservatively bounded by $\mathcal{N}(\mu,\frac{\sigma}{\sqrt{n}})$. As a consequence, the following inequality should be met:
\begin{align} 
\Pr\left({\frac{1}{n}\sum\limits_{k = 1}^n {{x_k}} \leq PL_L^{(n)}} \right) \leq {\left({1 + \epsilon } \right)^n}Q\left({\sqrt n \frac{{PL_L^{(n)} - \mu}}{\sigma }} \right)
\end{align}
Hence, the protection level $PL_L$ for a given \gls{IR} ( $ \text{IR} = (1+\epsilon)^n\, Q\!\left(\sqrt{n}\,\tfrac{PL_L^{(n)}-\mu}{\sigma}\right) $) is computed as:
\begin{equation}
    PL_L^{(n)} = {Q^{ - 1}}\left({\frac{{{\rm{IR}}}}{{{{\left({1 + \epsilon } \right)}^n}}}} \right)\frac{\sigma }{{\sqrt n }} + \mu
\end{equation}
We set the $IR$ to $10^{-3}$, consistent with typical misdetection probabilities in \gls{ARAIM} systems, evaluating $PL^1$ ($n=1$) and $PL^{10}$ ($n=10$) scenarios.

\textbf{Bonferroni-style additive bound vs convolution-based PL}

As a distribution-free alternative to convolution-based propagation, one may use a Bonferroni (union-bound) aggregation.
Allocating $\frac{IR}{n}$ to each term yields the one-sided bound
\begin{equation}
\begin{aligned}
PL^{(n)}_{L,\mathrm{Bonf}}
= {} & \frac{1}{n}\sum_{k=1}^{n}\left(\mu_k+\sigma_k\,Q^{-1}\!\left(\frac{\rm{IR}}{n}\right)\right) \\
= {} & {Q^{ - 1}}\left({\frac{{{\rm{IR}}}}{n}} \right)\sigma + \mu,
\label{eq:bonf_bound}
\end{aligned}
\end{equation}
This bound is conservative since:
\begin{align}
b_k &:= \mu_k+\sigma_k\,Q^{-1}\!\left(\tfrac{IR}{n}\right), \label{eq:bonf_bk}\\
&\Pr\!\left(\tfrac{1}{n}\sum_{k=1}^n x_k \le PL^{(n)}_{L,\mathrm{Bonf}}\right)
\le \Pr\!\left(\bigcup_{k=1}^n \{x_k \le b_k\}\right) \nonumber\\
&\le \sum_{k=1}^n \Pr(x_k \le b_k)
= \sum_{k=1}^n \tfrac{IR}{n}
= IR. \label{eq:bonf_proof}
\end{align}

In practice, $PL^{(n)}_{L,\mathrm{Bonf}}$ is typically more conservative (i.e., yields a larger magnitude bound) than the convolution-based $PL_L^{(n)}$ for $n\ge 2$ and $\rm{IR}<0.5$. Accordingly, we do not include interval-based \gls{UQ} baselines (e.g., standard \gls{CP}/\gls{CQR}) in our comparisons, since their multi-error aggregation typically relies on Bonferroni-style composition, yielding substantially more conservative bounds than convolution-based \glspl{PL}.

The left-tail Wasserstein distance ($W_L$) is computed as:
\begin{equation}
W_{L} =\sum_{\tau<\frac{1}{2}}\left|F^{-1}(\tau) - F_L^{-1}\left(\frac{\tau}{1+\epsilon}\right)\right|
\end{equation}
The quality indicator ($\mathcal{K}_L$) is calculated as the average of the global overbounding factor $K_\tau$ over the left-tail quantiles ($\mathcal{T}_L:=\tau <\frac{1}{2}$), where each $k_\tau$ is derived from the final overbounding distribution as defined in \eqref{eq20}:
\begin{equation}
    \mathcal{K}_L = \frac{1}{|\mathcal{T}_L|}\sum_{\tau<\frac{1}{2}}\left(\frac{(1+\epsilon)F_L(q_\tau)}{\tau}-1\right)
\end{equation}
And we use $|\mathcal{T}|=100$ quantile levels uniformly spaced in $(0,1)$ unless otherwise stated.

To capture a portion of epistemic uncertainty, we train an ensemble of $N_{\mathrm{ens}}=5$ models with identical architecture and hyperparameters but different random seeds.
For each member $j\in\{1,\dots,N_{\mathrm{ens}}\}$, the method outputs Gaussian bound parameters
$(\mu_L^{(j)},\sigma_L^{(j)})$ and $(\mu_R^{(j)},\sigma_R^{(j)})$, from which we compute the corresponding two-sided protection levels $(PL_L^{1,(j)},PL_R^{1,(j)})$.
We then select the most conservative parameters on each tail by
\begin{equation}
j_L^\star=\mathop{\arg\min}_{j\in\{1,\dots,N_{\mathrm{ens}}\}} PL_L^{1,(j)},\qquad
j_R^\star=\mathop{\arg\max}_{j\in\{1,\dots,N_{\mathrm{ens}}\}} PL_R^{1,(j)}.
\end{equation}
Accordingly, we report $(\mu_L,\sigma_L)=(\mu_L^{(j_L^\star)},\sigma_L^{(j_L^\star)})$ and $(\mu_R,\sigma_R)=(\mu_R^{(j_R^\star)},\sigma_R^{(j_R^\star)})$.
The same ensemble protocol is applied to both the synthetic simulations and the real-world experiments.

\subsection{Ablation Study}\label{sec4_2}
An ablation study was conducted using Type 1 distribution to analyze the contribution of individual components. Hyperparameters were kept consistent with previous simulations, unless explicitly stated otherwise. Five key aspects were examined:
\begin{enumerate}
\item Comparison of cases 1, 2, and 3 to verify the validity of constraints and parameterizations.
\item Assessment of the Wasserstein penalty term's impact on the loss function.
\item Performance evaluation comparing scenarios with and without quantile weighting factor $w_\tau$.
\item Influence of different scaling factors $t$ (with values $t=1$, $t=1-50\lambda$, and $t=1-200\lambda$).
\item Analysis of the learnable versus fixed local overbounding factor $k_\tau$.
\end{enumerate}

\subsection{Real world experiments}\label{sec4_3}
To demonstrate the practical effectiveness of our proposed method, we conducted experiments across three real-world error scenarios:
\subsubsection{Troposphere residual errors}
We analyzed data from the Hong Kong reference station “HKSL” throughout 2022. Tropospheric residual errors were computed as the difference between estimated and ground truth \gls{ZWD} values. The ground truth \gls{ZWD} was obtained by subtracting the dry delay component from total zenith delay values derived from \gls{IGS} post-processed products downloaded via the \gls{CDDIS}\cite{IGS_TROP}. The estimated \gls{ZWD} was obtained by applying \gls{GPR} to interpolate HKSL's \gls{ZWD} using \gls{PPP}-derived \gls{ZWD} values from other Hong Kong reference stations\cite{hou2024interpolation}. Given the relatively small magnitude of these residual errors, all values were converted from meters to centimeters to facilitate stable model training. The tropospheric residual errors exhibited a symmetric and unimodal distribution, satisfying the symmetric overbounding assumption. Under these conditions, applying full constraints across all quantile levels, as required by paired overbounding, introduces unnecessary conservatism. Therefore, we introduced a modified version of our approach—termed "Case 2-Half Constrained"—which enforces conservative constraints only for the lower half of the quantile range ($\tau<\frac{1}{2}$). We compared our methods ("Case 2" and "Case 2-Half Constrained") against paired overbounding, two-step Gaussian overbounding, and quantile overbounding methods.
\subsubsection{Multipath errors (UrbanNav dataset)}
We used the UrbanNav dataset~\cite{hsu2023hong}, containing raw \gls{GNSS} measurements and multi-sensor data collected via a Honda electric vehicle operating in urban canyons. Ground-truth multipath errors were extracted using the approach presented in~\cite{liu2024overbounding}. To ensure consistency, we employed the same \gls{GNSS} features (elevation angle, azimuth angle, \gls{CN0}) and neural network backbone (MPNMLP) as~\cite{liu2024overbounding}, excluding the time-series (MPNLSTM) model for simplicity.
Multipath errors exhibit distinct conditional distributions at each time step due to changing input features. To visualize the data uniformly, we normalized each residual using the model-predicted parameters: $y_{norm} = \tfrac{y-\mu}{\sigma}$. After normalization, the distribution of $y_{norm}$ should ideally be conservatively bounded by a standard normal distribution. To prevent extremely small predicted standard deviations from inflating $y_{norm}$, we introduce a minimum threshold via the transformation: $\sigma = \sigma_{\min}+\exp(\log_{\sigma})$. The metrics $PL^1$ and $PL^{10}$ were computed as the averages over all time steps. The overbounding factor $\mathcal{K}$ was calculated after normalization. Since $\sigma$ varies across time steps, the Wasserstein distance of $y_{\text{norm}}$ will have different weights for each epoch, which is meaningful, and we therefore omitted this metric from this evaluation. Our results were compared against quantile overbounding (using identical prior features), as well as two-step Gaussian and paired overbounding methods applied to the unconditional multipath error distribution due to the lack of full conditional distribution knowledge.
\subsubsection{Ionosphere residual errors (Hong Kong)}
We evaluated our method using ionospheric residual errors obtained from \gls{TEC} data measured by 19 Continuously Operating Reference Stations (CORS) across Hong Kong. The data focused on extreme ionospheric disturbances occurring from January 1, 2020, to June 30, 2025, corresponding to Solar Cycle 25, identified through geomagnetic index thresholds ($K_p \geq 8$ and $Dst < -100 nT$), yielding over 3,000,000 samples\cite{liu2023ionospheric-1d6, liu2024ionospheric}. Residual ionospheric delays and their safety bounds were calculated via universal Kriging as a baseline~\cite{walter2001robust, sparks2011estimating, ren2024anisotropic}. In traditional \gls{SBAS}, the ionospheric error bound (\gls{GIVE}) is composed of two uncertainty components: the sampled uncertainty ($\sigma_{sampled}$), derived during delay estimation, and the unsampled uncertainty ($\sigma_{unsampled}$), tabulated using discretized bins of \gls{Rfit} and \gls{RCM} (40 bins total).
To replicate and enhance this structure, we directly input continuous values of \gls{Rfit} and \gls{RCM} into a five-layer \gls{MLP} network (hidden layers: 512, 1024, 2048, 1024, 512 units) to learn a conditional overbounding distribution. Predictions were similarly normalized: $y_{norm} = \tfrac{y-\mu}{\sigma}$. Given that the \gls{SBAS} methodology assumes symmetric, zero-mean, unimodal error distributions, we tested the symmetric (half-constrained) variant of our proposed overbounding loss ("Case 2-Half Constrained"). Additionally, to assess the value of supplementary features, we incorporated elevation angle as an extra input (denoted "Case 2-Half+Elev"). We compared our results against two-step Gaussian and paired overbounding methods applied to the overall (unconditional) error distribution. Metrics used for evaluation matched those in the multipath error scenario.

Fig. \ref{fig:real_data_distributions} shows the empirical total error distributions (without prior conditioning) for each scenario. Tropospheric residuals appear near-Gaussian; multipath errors are positively skewed with multimodal distributions, while ionospheric residuals resemble Laplace distributions, highlighting diverse practical conditions addressed by our method.
\begin{figure*}[!htbp]
    \centering
    \subfloat[]{\includegraphics[width=0.3\textwidth]{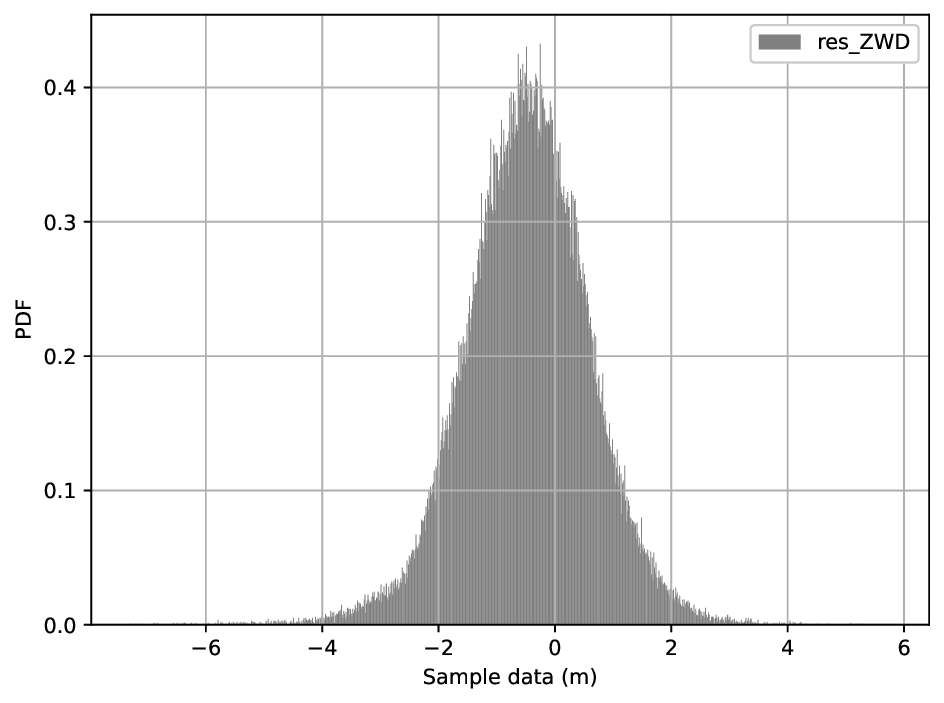}}
    \hfill
    \subfloat[]{\includegraphics[width=0.3\textwidth]{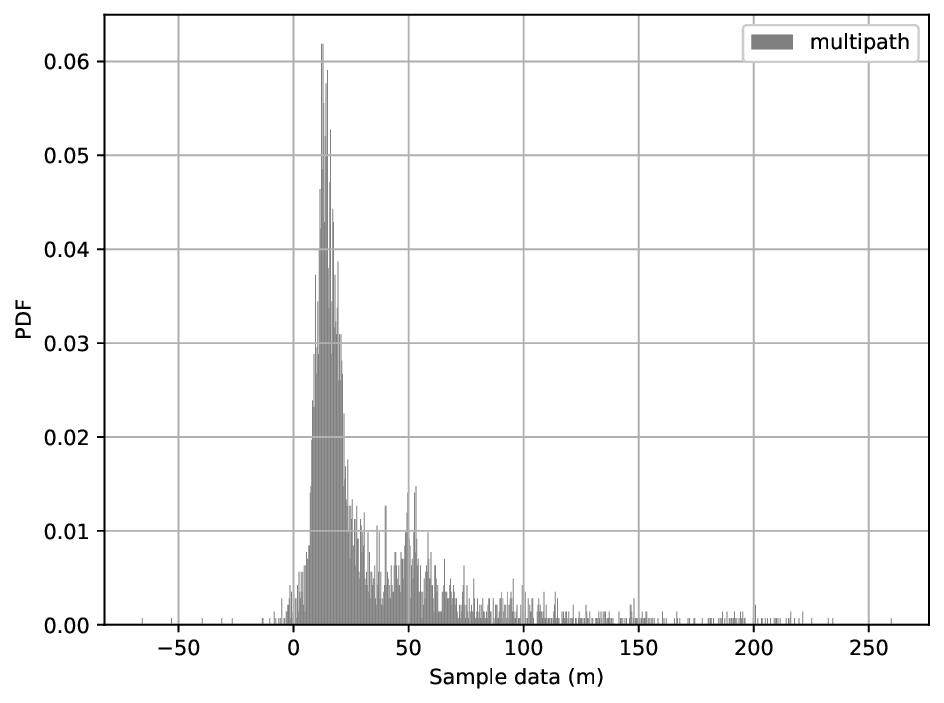}}
    \hfill
    \subfloat[]{\includegraphics[width=0.3\textwidth]{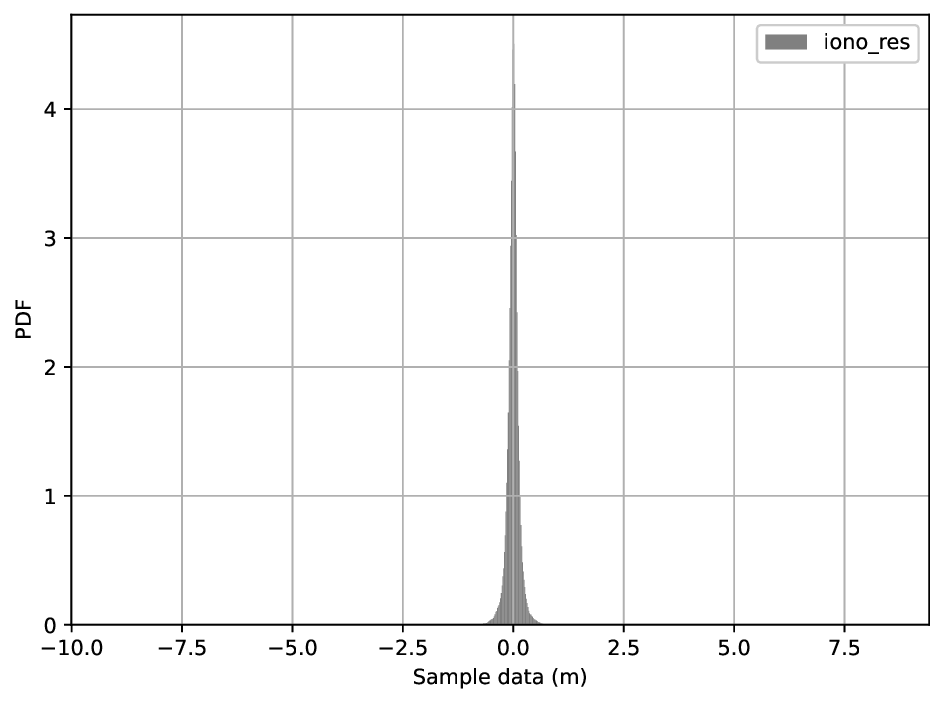}}
    \caption{Histograms of three representative real-world error distributions: 
    (a) ZWD residual error, (b) GNSS multipath error, and (c) ionospheric delay residual error.}
    \label{fig:real_data_distributions}
\end{figure*}

\section{Results and analysis\label{sec5}}
This section presents the results of simulations designed to compare the proposed method against traditional overbounding techniques. An ablation study is conducted to analyze the impact of key components, followed by empirical evaluations on three real-world datasets.
\subsection{Simulation Results}\label{sec5_1}
Fig.~\ref{fig:dist1_result} presents the \gls{PDF}, \gls{QQ} plot, and logarithmic-scale \gls{CDF} of the Type 1 distribution, along with the corresponding left bounds obtained using our proposed method (Case 2), paired overbounding (Paired OB), two-step Gaussian overbounding (Two-step OB), and quantile overbounding (Quantile OB).

In Fig.~\ref{fig:dist1_result}(a), the proposed method (blue line) closely aligns with the empirical distribution on the left side, demonstrating a noticeably tighter fit compared to the other methods. The two-step Gaussian overbounding method (green line) shows a looser approximation, while paired overbounding (orange line) and quantile overbounding (red line) produce distributions with significantly larger standard deviations. Fig.~\ref{fig:dist1_result}(b)–(c) further confirm that for the region where $\mathrm{CDF} < 0.5$, our method not only provides a closer approximation to the empirical distribution but also consistently stays above it, ensuring conservatism. The blue crosses represent the quantile points used during training. Notably, the quantile overbounding curve intersects the empirical distribution, indicating its failure to maintain conservatism—primarily due to its reliance on a single quantile estimate and the assumption of zero-mean Gaussianity, which is generally unsuitable for this type of distribution. Additionally, in the \gls{QQ} plot (Fig.~\ref{fig:dist1_result}(b)), the non-linear behavior observed in the tail regions reflects the application of the excess mass technique: the quantile at level $\tau$ is adjusted to $\tfrac{\tau}{1+\epsilon}$, causing deviation from a straight line. The area between the method curve and the sample distribution curve for quantiles below zero, partially reflects the Wasserstein distance $W_L$. In this aspect, our method shows the smallest area, suggesting the tightest fit among all approaches.

\begin{figure*}[!htbp]
    \centering
    \subfloat[]{\includegraphics[width=0.33\textwidth]{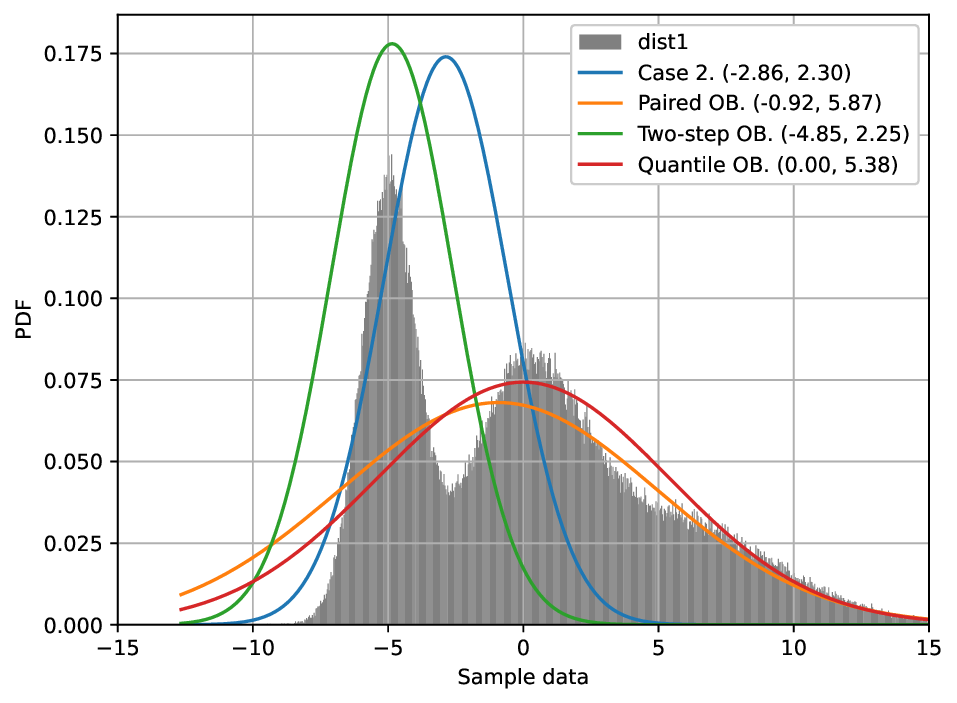}}
    \hfill
    \subfloat[]{\includegraphics[width=0.33\textwidth]{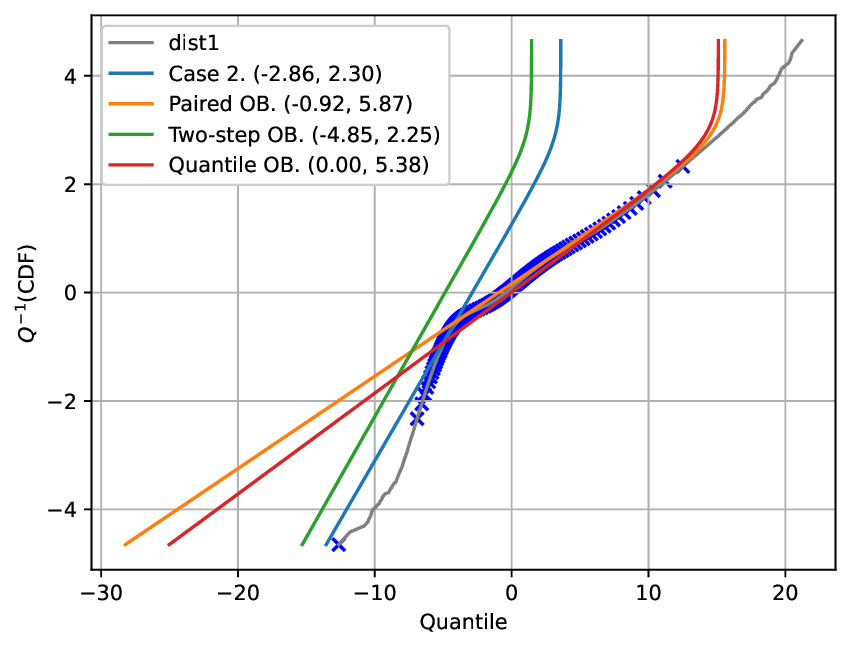}}
    \hfill
    \subfloat[]{\includegraphics[width=0.33\textwidth]{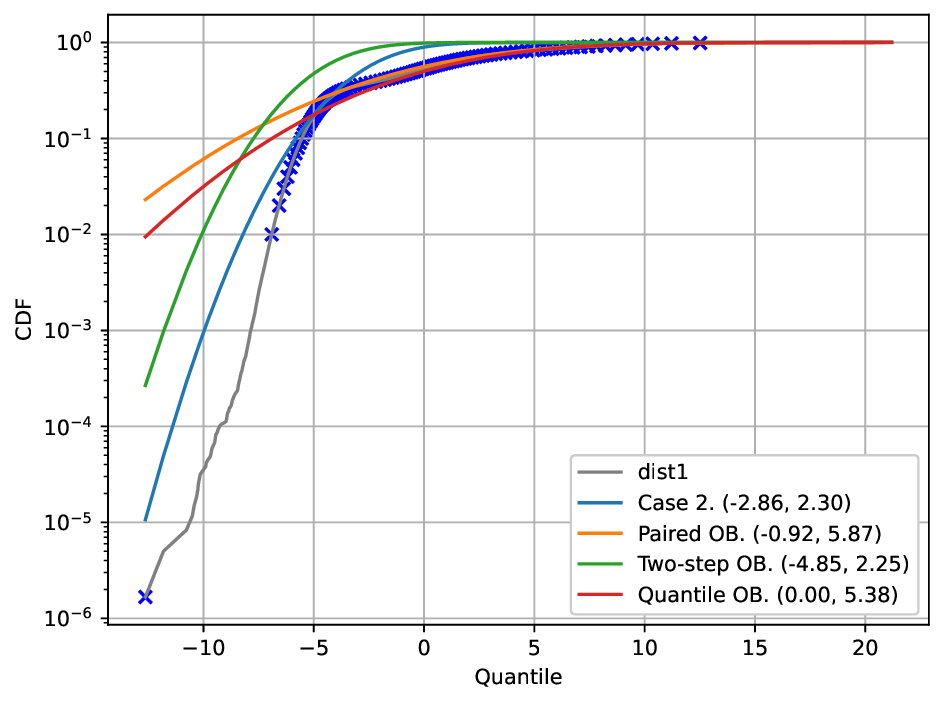}}
    \caption{Visualization of left bounds for Type 1 distribution using various methods: 
    (a) PDF comparison, (b) QQ plot, and (c) CDF in log scale.}
    \label{fig:dist1_result}
\end{figure*}

Table~\ref{tab:dist1_result} summarizes quantitative results. "OB" denotes whether overbounding is achieved; if not, it specifies which part violates the conservative requirement. When overbounding is successful—where the quantile overbounding method fails—our proposed method consistently outperforms both the two-step and paired overbounding approaches across all evaluation metrics.
Specifically, considering the known distribution, the ground truth interval for $PL^1$ is [-7.856, 15.986], with its 10-fold self-convolution $PL^{10}$ within [-4.192, 5.065]. Compared to the two-step overbounding method ($PL^1$: [-11.798, 18.999], $PL^{10}$: [-7.054, 8.86]) and paired overbounding ($PL^1$: [-19.074, 17.424], $PL^{10}$: [-6.671, 8.229]), our method significantly reduces the protection level range.
For $PL^1$, our method narrows the interval to [-11.069, 17.490], reducing the range length by approximately 7.3\% compared to two-step and 21.7\% compared to paired overbounding. For $PL^{10}$, our interval of [-5.242, 7.863] achieves approximately 18\% reduction over two-step and 12\% reduction over paired overbounding.
Furthermore, compared to the ground truth, our method substantially reduces deviation: the total $PL^1$ deviation drops from 7.0 (two-step) and 12.6 (paired) to 4.71—over a 33\% improvement. For $PL^{10}$, deviation improves from 6.6 (two-step) and 5.6 (paired) to 3.8, representing approximately a 32\% reduction.
Regarding Wasserstein distance and the average overbounding factor $\mathcal{K}$, our method achieves notable improvements. For the left bound, $W_L$ is reduced by 71\% and 58\% compared to two-step and paired overbounding, respectively. The overbounding factor $\mathcal{K}_L$ is reduced by 65\% and 99\% compared to two-step and paired overbounding. For the right bound, $W_R$ is reduced by 24\% and 11\%, and $\mathcal{K}_R$ decreases by 58\% and 5\% relative to the two baselines. These results demonstrate that our approach achieves a tighter overbounding distribution, aligning much closer to the true distribution than previous methods, while preserving the required conservatism. 

\begin{table*}[!htbp]
    \centering
    \caption{Results on Type 1 Distribution}
    \label{tab:dist1_result}
    \resizebox{\linewidth}{!}{
\begin{tabular}{l|l l l l l l l|l l l l l l l}
\hline
Method                & $\mu_L$ & $\sigma_L$ & OB & $PL_L^1$ & $PL_L^{10}$ & $W_L$ & $\mathcal{K}_L$ & $\mu_R$ & $\sigma_R$ & OB & $PL_R^1$ & $PL_R^{10}$ & $W_R$ & $\mathcal{K}_R $\\ \hline
Paired OB   & -0.917  & 5.874      & True         & -19.074 & -6.671   & 91.689  & 277.874 & 3.963   & 4.355      & True         & 17.424  & 8.229    & 173.297 & 1.609  \\ 
Two-step OB & -4.853  & 2.247      & True         & -11.798 & -7.054   & 131.890 & 5.548   & 4.156   & 4.802      & True         & 18.999  & 8.860    & 202.263 & 3.667  \\ 
Quantile OB  & 0.000   & 5.381      & 0.19-0.84    & -16.631 & -5.271   & 67.336  & 113.807 & 0       & 2.969      & 0-0.4        & 9.179   & 2.909    & 88.247  & -0.450 \\ 
Case 2                & -2.539  & 2.760      & True         & -11.069 & -5.242   & 38.323  & 1.967   & 3.397   & 4.560      & True         & 17.490  & 7.863    & 153.816 & 1.528  \\ \hline
\end{tabular}
}
\end{table*}

Fig.~\ref{fig:dist2_result} presents the visualization results for the Type 2 distribution, including the \gls{PDF}, \gls{QQ} plot, and logarithmic \gls{CDF} of the left bounds generated by various overbounding methods. It is evident that our proposed method produces a distribution that more closely aligns with the actual one. As shown in Fig.~\ref{fig:dist2_result}(c), the two-step Gaussian overbounding method deviates significantly from the true distribution, especially in the left tail, while our method (blue line) aligns more closely than the paired overbounding method (orange line). Notably, in this distribution, paired overbounding shows performance comparable to our method. This is primarily because both methods share similar constraint formulations: paired overbounding initializes its search from the distribution’s mean, and the first parameter set that meets conservatism criteria tends to be near the optimized parameters obtained through our framework.
\begin{figure*}[!htbp]
    \centering
    \subfloat[]{\includegraphics[width=0.33\textwidth]{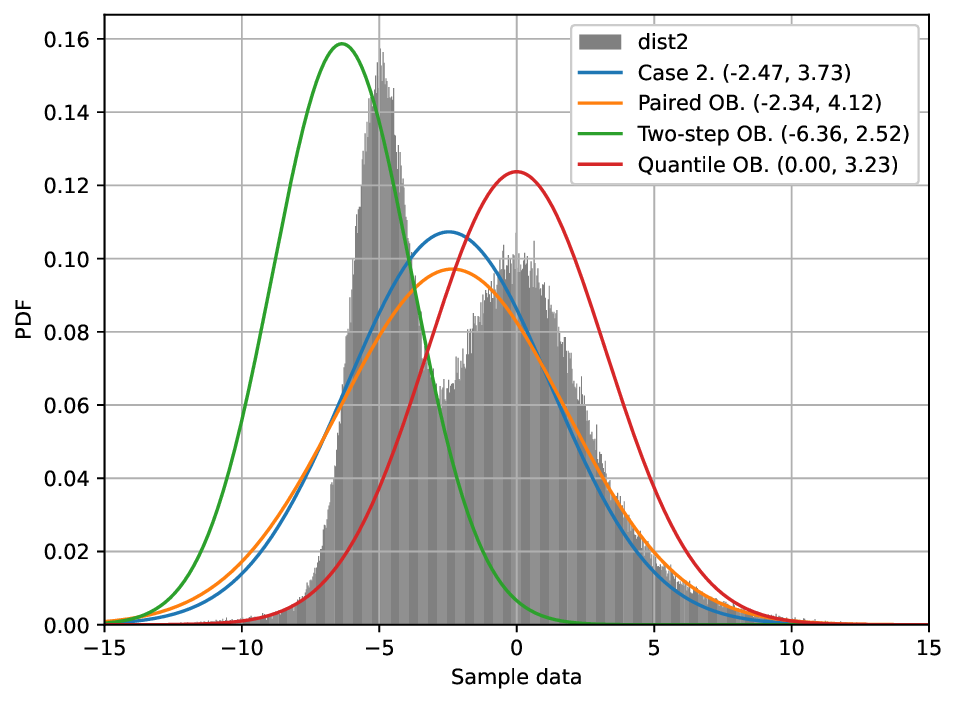}}
    \hfill
    \subfloat[]{\includegraphics[width=0.33\textwidth]{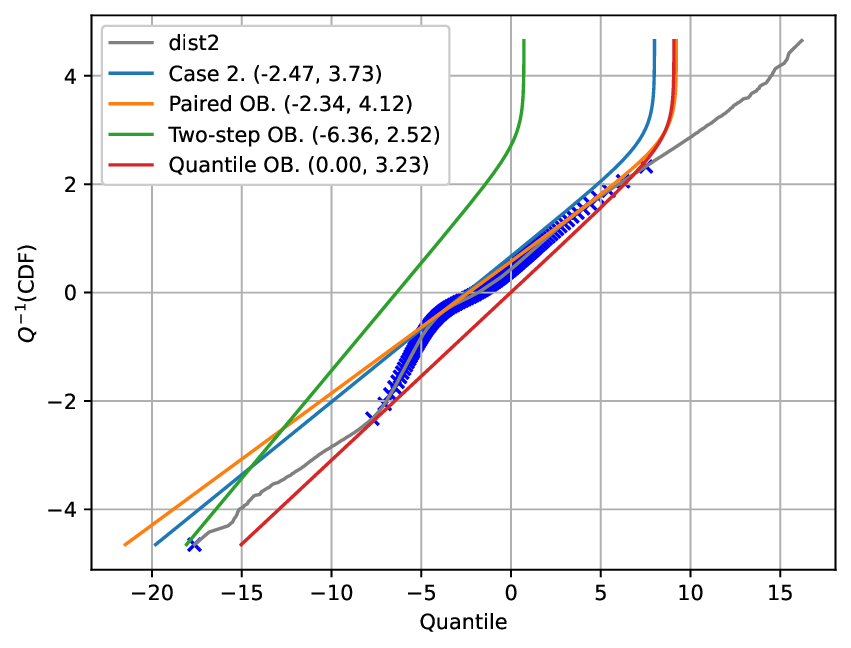}}
    \hfill
    \subfloat[]{\includegraphics[width=0.33\textwidth]{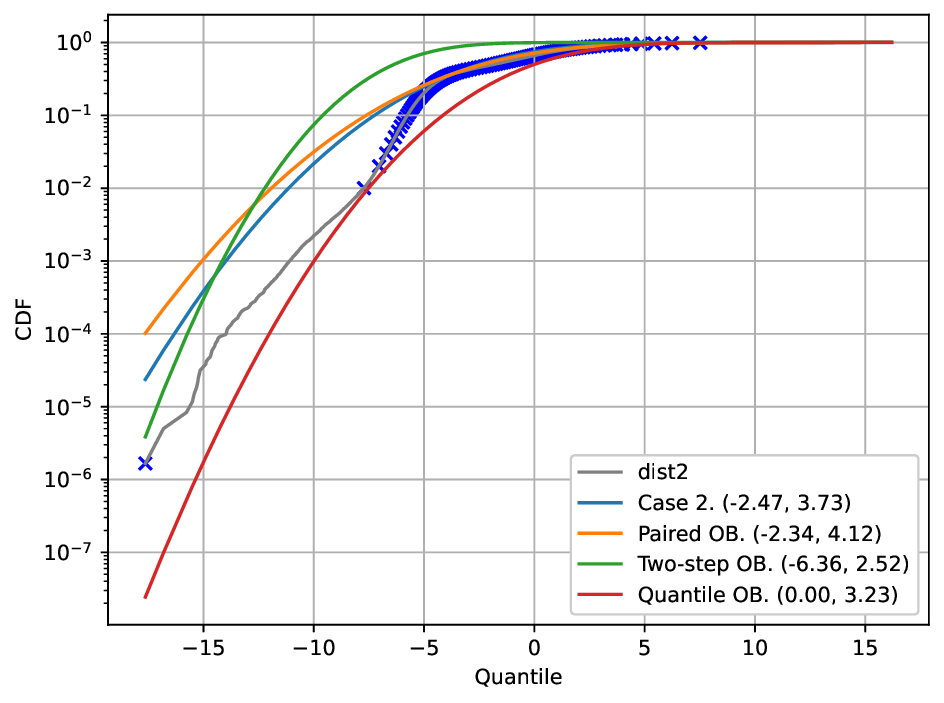}}
    \caption{Visualization of left bounds for Type 2 distribution using various methods: 
    (a) PDF comparison, (b) QQ plot, and (c) CDF in log scale.}
    \label{fig:dist2_result}
\end{figure*}

\begin{table*}[!htbp]
    \centering
    \caption{Results on Type 2 Distribution}
    \label{tab:dist2_result}
    \resizebox{\linewidth}{!}{
\begin{tabular}{l|l l l l l l l|l l l l l l l}
\hline
Method                & $\mu_L$ & $\sigma_L$ & OB & $PL_L^1$ & $PL_L^{10}$ & $W_L$ & $\mathcal{K}_L$ & $\mu_R$ & $\sigma_R$ & OB & $PL_R^1$ & $PL_R^{10}$ & $W_R$ & $\mathcal{K}_R $\\ \hline
Paired OB   & -2.344  & 4.117      & True         & -15.069 & -6.377   & 51.626  & 1.961  & 0.996   & 3.904      & True         & 13.063  & 4.820    & 141.365 & 1.930 \\ 
Two-step OB & -6.360  & 2.521      & True         & -14.150 & -8.829   & 183.499 & 3.812  & 0.280   & 4.140      & True         & 13.075  & 4.335    & 115.705 & 1.732 \\ 
Quantile OB  & 0.000   & 3.232      & 0-0.98       & -9.991  & -3.166   & 103.759 & -0.617 & 0.000   & 3.319      & 0-2e-6        & 10.259  & 3.251    & 67.828  & 0.585 \\ 
Case 2                & -2.471  & 3.727      & True         & -13.991 & -6.122   & 41.137  & 0.904  & 0.414   & 3.542      & True         & 11.363  & 3.884    & 96.644  & 0.931 \\ \hline
\end{tabular}
}
\end{table*}

Table~\ref{tab:dist2_result} quantitatively supports these observations. When overbounding is successful—meaning the quantile overbounding method fails in the left tail—our approach consistently outperforms both two-step and paired overbounding methods across multiple metrics. Specifically, for the left bound, $W_L$ is reduced by 77\% and 20\%, and the average overbounding factor $\mathcal{K}_L$ improves by 76\% and 54\% compared to the two-step and paired methods, respectively. For the right bound, $W_R$ is reduced by 16\% and 31\%, while $\mathcal{K}_R$ decreases by 46\% and 51\%.
In terms of protection levels, the $PL^1$ range ($PL^1_R - PL^1_L$) improves by 7\% and 10\%, and the $PL^{10}$ range improves by 24\% and 10.7\% compared to two-step and paired overbounding, respectively. Furthermore, when compared to the ground truth $PL^1$ and $PL^{10}$ ranges of [-11.093, 10.986] and [-4.815, 1.944], our method reduces deviation substantially: $PL^1$ total deviation drops from 5.15 (two-step) and 6.05 (paired) to 3.3—a 36\% improvement. Similarly, $PL^{10}$ deviation improves from 6.4 (two-step) and 4.4 (paired) to 3.2, achieving a 27\% reduction.
These results confirm that our method produces tighter bounds while rigorously maintaining conservatism. Notably, in negatively biased, multi-modal distributions such as Type 2, the two-step overbounding method performs worse than paired overbounding. In contrast, our method consistently outperforms both, offering tighter and more reliable overbounding—further highlighting its practical advantage in complex real-world scenarios.

Similarly, Fig.~\ref{fig:dist3_result} shows the left-bound results for the Type 3 distribution. Across all visualizations, our method produces a distribution that more closely aligns with the true distribution. In particular, subfigure (c) demonstrates that the two-step Gaussian overbounding method performs substantially worse than both our method and paired overbounding, especially in the left tail. While paired overbounding shows relatively similar performance to our approach, our method (blue line) remains consistently closer to the true distribution than paired overbounding (orange line), particularly on the left side.
\begin{figure*}[!htbp]
    \centering
    \subfloat[]{\includegraphics[width=0.33\textwidth]{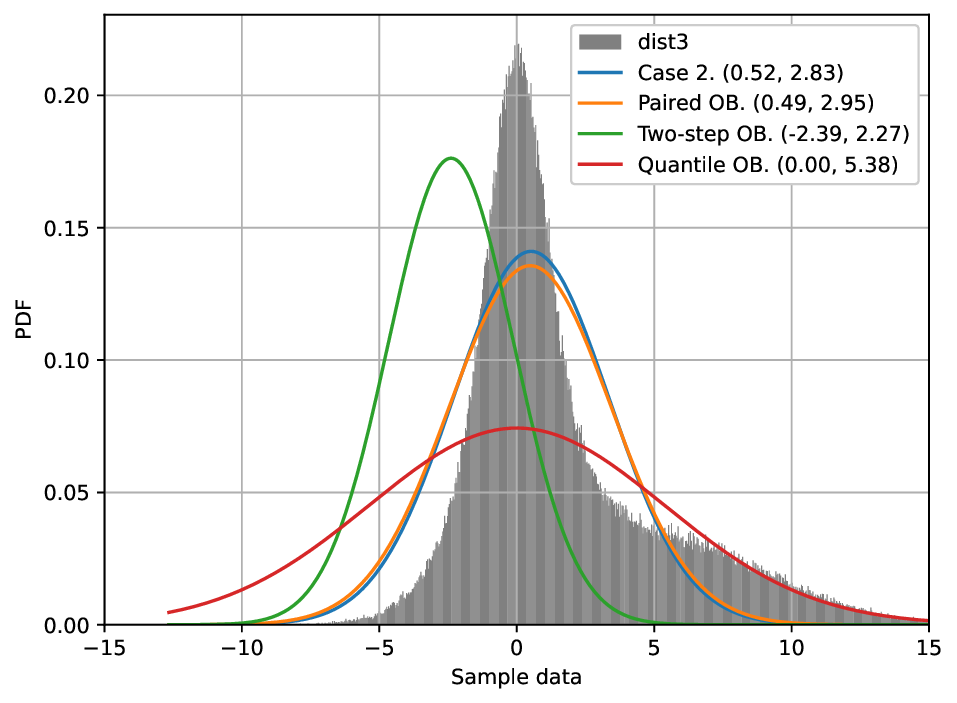}}
    \hfill
    \subfloat[]{\includegraphics[width=0.33\textwidth]{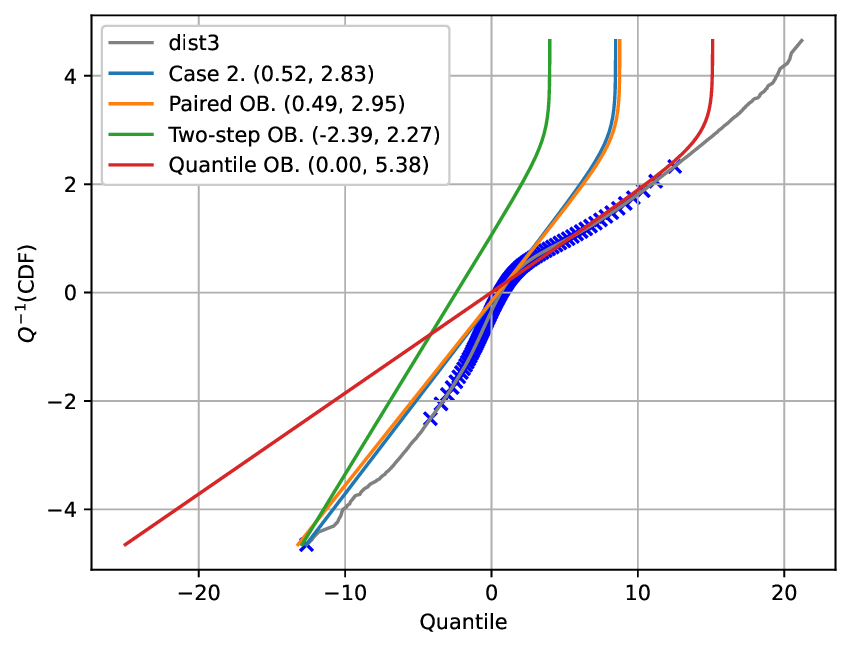}}
    \hfill
    \subfloat[]{\includegraphics[width=0.33\textwidth]{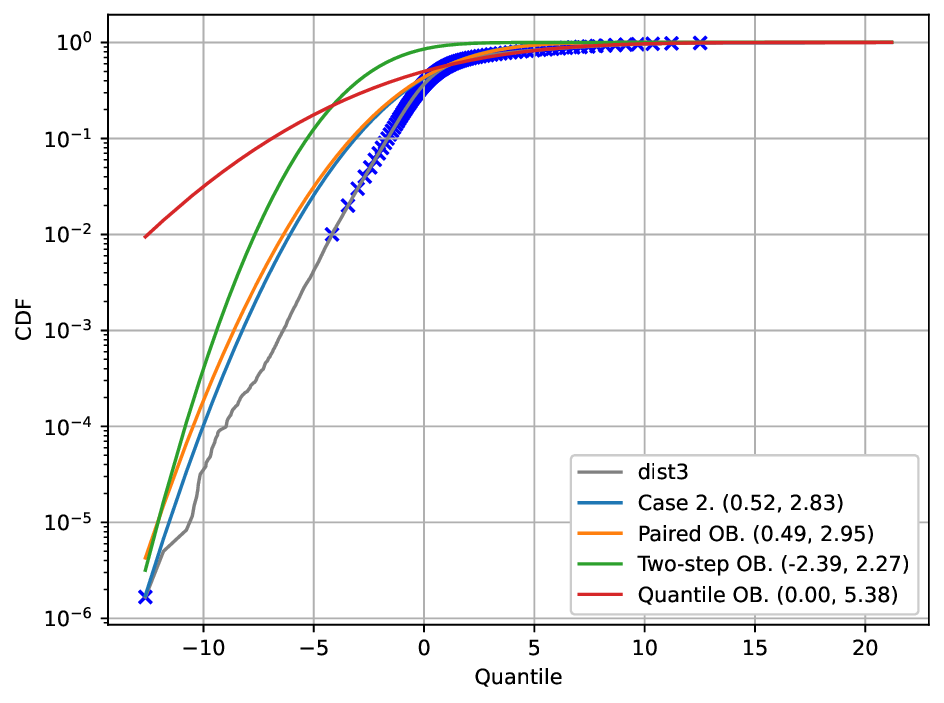}}
    \caption{Visualization of left bounds for Type 3 distribution using various methods:
    (a) PDF comparison, (b) QQ plot, and (c) CDF (log scale).}
    \label{fig:dist3_result}
\end{figure*}

\begin{table*}[!htbp]
    \centering
    \caption{Results on Type 3 Distribution}
    \label{tab:dist3_result}
    \resizebox{\linewidth}{!}{
\begin{tabular}{l|l l l l l l l|l l l l l l l}
\hline
Method                & $\mu_L$ & $\sigma_L$ & OB & $PL_L^1$ & $PL_L^{10}$ & $W_L$ & $\mathcal{K}_L$ & $\mu_R$ & $\sigma_R$ & OB & $PL_R^1$ & $PL_R^{10}$ & $W_R$ & $\mathcal{K}_R $\\ \hline
Paired OB   & 0.490   & 2.948      & True         & -8.623  & -2.398   & 50.333  & 0.833  & 5.170   & 4.087      & True         & 17.802 & 9.173   & 212.630 & 2.169 \\ 
Two-step OB & -2.392  & 2.269      & True         & -9.405  & -4.614   & 165.082 & 3.620  & 5.227   & 4.179      & True         & 18.145 & 9.321   & 219.485 & 2.487 \\ 
Quantile OB  & 0.000   & 5.381      & 0.59-0.84    & -16.631 & -5.271   & 178.401 & 11.624 & 0.000   & 1.794      & 0-0.98       & 5.546  & 1.758   & 31.598  & 0.471 \\ 
Case 2                & 0.522   & 2.835      & True         & -8.240  & -2.255   & 43.808  & 0.713  & 4.759& 3.682      & True         & 16.141 & 8.366   & 174.631 & 1.402 \\ \hline
\end{tabular}
}
\end{table*}

This observation is quantitatively supported by the results in Table~\ref{tab:dist3_result}. When overbounding is successful, our method exhibits consistent improvements across all evaluation metrics. For the left bound, the Wasserstein distance $W_L$ is reduced by 73\% and 13\%, and the average overbounding factor $\mathcal{K}_L$ is improved by 80\% and 14\%, compared to the two-step and paired methods, respectively. For the right bound, $W_R$ decreases by 20\% and 18\%, and $\mathcal{K}_R$ improves by 44\% and 35\%, respectively. In terms of protection levels, the range of $PL^1$ improves by 11.5\% and 8\%, while the range of $PL^{10}$ improves by 24\% and 8\%, relative to the two-step and paired methods.
Compared to the ground truth $PL^1$ and $PL^{10}$ ranges of [–6.388, 15.986] and [–1.188, 5.723], our method reduces total deviation notably: $PL^1$ deviation drops from 5.17 (two-step) and 4.1 (paired) to just 2, representing more than a 51\% improvement. $PL^{10}$ deviation is reduced from 7 (two-step) and 4.7 (paired) to 3.7, achieving approximately a 21\% reduction.
These results highlight the superior fidelity of the proposed method in modeling heavy-tailed distributions. 

\subsection{Ablation Study}\label{sec5_2}
\paragraph{Comparison of Case 1, Case 2, and Case 3}
To validate the theoretical derivation under different parameter constraints, we evaluate three cases on the left tail of the Type 1 distribution. The results are shown in Table~\ref{tab:ablation_cases}. It can be observed that Case 1 and Case 3 fail to fully overbound the true distribution, while Case 2 successfully satisfies the overbounding condition. Notably, Case 3 performs better than Case 1, with only a small quantile range violating the overbounding requirement. Case 2, with both shape ($\sigma$) and scale ($k$) shared across quantiles, achieves full conservatism while maintaining competitive performance metrics.
\begin{table}[!htbp]
    \centering
    \caption{Comparison of different parameter constraint cases on Type 1 distribution (left bound)}
    \label{tab:ablation_cases}
    \resizebox{\columnwidth}{!}{
\begin{tabular}{l|lllllll}
\hline
Method                & $\mu_L$ & $\sigma_L$ & OB & $PL_L^1$ & $PL_L^{10}$ & $W_L$ & $\mathcal{K}_L$ \\ \hline
case1  & -0.571 & 496.679 & 0.51-0.99    & -1535.790 & -487.097 & 20790.330 & 102.284 \\ 
case2  & -2.539  & 2.760      & True         & -11.069 & -5.242   & 38.323  & 1.967   \\ 
case3  & -1.087 & 5.444   & 0.32-0.33    & -17.915   & -6.420   & 81.645    & 204.748 \\ \hline
\end{tabular}
}
\end{table}

\paragraph{Effect of Wasserstein Distance Penalty Term}
To assess the effect of the Wasserstein distance penalty $\mathcal{J}_p$, we compare results with and without its inclusion in the loss function. Table~\ref{tab:ablation_penalty} shows that adding $\mathcal{J}_p$ significantly improves key metrics by 40\%, 83\%, and 72\% for $PL_L^{10}$, $W_L$, and $\mathcal{K}_L$, respectively, while maintaining a minimal performance loss (-0.7\%) in $PL_L^1$.
\begin{table}[!htbp]
    \centering
    \caption{Effect of including Wasserstein distance penalty on Type 1 distribution (left bound)}
    \label{tab:ablation_penalty}
    \resizebox{\columnwidth}{!}{
\begin{tabular}{l|lllllll}
\hline
Method                & $\mu_L$ & $\sigma_L$ & OB & $PL_L^1$ & $PL_L^{10}$ & $W_L$ & $\mathcal{K}_L$ \\ \hline
With $\mathcal{J}_p$  & -2.539  & 2.760      & True         & -11.069 & -5.242   & 38.323  & 1.967 \\ 
Without $\mathcal{J}_p$ & -7.678  & 1.071      & True         & -10.987 & -8.726   & 222.364 & 7.173 \\ \hline
\end{tabular}}
\end{table}

\paragraph{Impact of Quantile Weights $w_\tau$}
We examine the impact of the weighting scheme $w_\tau$ by comparing our designed weighting with uniform weighting. Table~\ref{tab:ablation_weights} shows that using the proposed $w_\tau$ ensures conservatism across the quantile range, particularly in the tail. In contrast, uniform weighting leads to underbounding in extreme quantiles (notably for $\tau < 10^{-6}$), indicating insufficient conservatism.
\begin{table}[!htbp]
    \centering
    \caption{Impact of quantile weighting scheme $w_\tau$}
    \label{tab:ablation_weights}
    \resizebox{\columnwidth}{!}{
\begin{tabular}{l|lllllll}
\hline
Method                & $\mu_L$ & $\sigma_L$ & OB & $PL_L^1$ & $PL_L^{10}$ & $W_L$ & $\mathcal{K}_L$ \\ \hline
With $w_\tau$   & -2.539  & 2.760      & True         & -11.069 & -5.242   & 38.323 & 1.967 \\ 
Without $w_\tau$& -3.187  & 1.907      & 0-1e-6    & -9.082  & -5.055   & 35.093 & 0.363 \\ \hline
\end{tabular}
}
\end{table}

\paragraph{Effect of Scaling Factor $t$}
To verify the correctness of our theoretical derivation for the scaling factor $t$, we evaluate different values of $t$ on the right tail of the type 1 distribution. As shown in Table~\ref{tab:ablation_t}, only the value derived from our theory ($t=1-200\lambda$) maintains full conservatism, while the other settings fail to meet the overbounding condition in the extreme tail.
\begin{table}[!htbp]
    \centering
    \caption{Effect of scaling factor $t$ on type 1 distribution (right side)}
    \label{tab:ablation_t}
    \resizebox{\columnwidth}{!}{
\begin{tabular}{l|lllllll}
\hline
Method                & $\mu_R$ & $\sigma_R$ & OB & $PL_R^1$ & $PL_R^{10}$ & $W_R$ & $\mathcal{K}_R$ \\ \hline
$t=1$  & -2.489 & 2.835   & 0-0.01       & -15.683 & -6.919 & 109.015 & 0.711 \\ 
$t=1-50\lambda$  & -2.857 & 4.150   & 0-0.01       & -15.684 & -6.922 & 109.164 & 0.713 \\ 
$t=1-200\lambda$ & -3.397 & 4.560   & True         & -17.490 & -7.863 & 153.816 & 1.528 \\ \hline
\end{tabular}
}
\end{table}

\paragraph{Effect of Learnable Shape Parameter $k$}
Finally, we investigate the role of a learnable shape parameter $k$ by comparing it to a fixed $k=1$ setting. Table~\ref{tab:ablation_k} demonstrates that enabling $k$ to be learnable significantly improves performance across all metrics. This is because learnable $k$ provides additional flexibility, enabling better optimization behavior and stability. Although using a very large batch size (e.g., the entire dataset) can mitigate this issue when $k$ is fixed, this is impractical for large-scale datasets.
\begin{table}[!htbp]
    \centering
    \caption{Effect of learnable vs fixed shape parameter $k$}
    \label{tab:ablation_k}
    \resizebox{\columnwidth}{!}{
\begin{tabular}{l|lllllll}
\hline
Method                & $\mu_L$ & $\sigma_L$ & OB & $PL_L^1$ & $PL_L^{10}$ & $W_L$ & $\mathcal{K}_L$ \\ \hline
learnable $k$    & -2.539 & 2.760   & True         & -11.069 & -5.242 & 38.323 & 1.967   \\ 
fixed $k$ & -0.917 & 5.874   & True         & -19.074 & -6.671 & 90.855 & 140.504 \\ \hline
\end{tabular}
}
\end{table}

\subsection{Empirical Experiments\label{sec5_3}}
We further evaluated the practical applicability of our proposed method across three real-world scenarios, involving troposphere residual errors, multipath errors, and ionosphere residual errors.

Fig.~\ref{fig:zwd_result} illustrates the \gls{PDF}, \gls{QQ} plot, and logarithmic-scale \gls{CDF} of the \gls{ZWD} residual errors for left bounds. As indicated in Fig.~\ref{fig:zwd_result}(a), the empirical residual errors (gray line) exhibit a symmetric and unimodal distribution. Although quantile overbounding (purple line) closely approximates the empirical distribution, it fails to maintain conservatism. In contrast, Case 2 (blue line) and paired overbounding (green line) provide slight overbounds, closely aligning with each other. Meanwhile, two-step Gaussian overbounding (red line) and our half-constrained variant of Case 2 (orange line) generate tighter bounds, with the half-constrained version offering the closest conformity to the empirical distribution from Fig.~\ref{fig:zwd_result}(b).
\begin{figure*}[!htbp]
    \centering
    \subfloat[]{\includegraphics[width=0.33\textwidth]{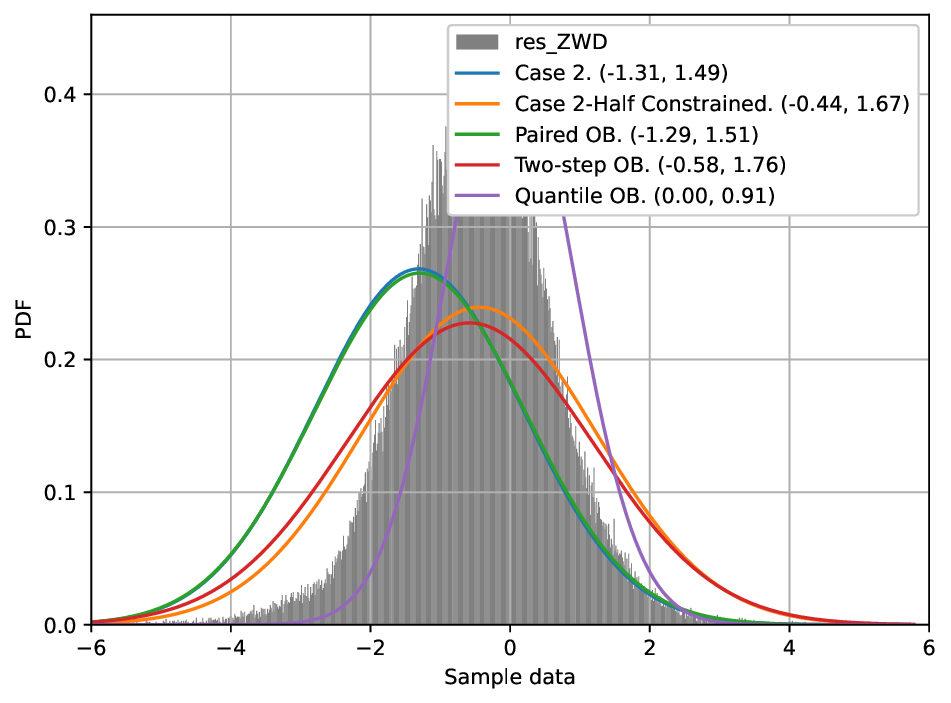}}
    \hfill
    \subfloat[]{\includegraphics[width=0.33\textwidth]{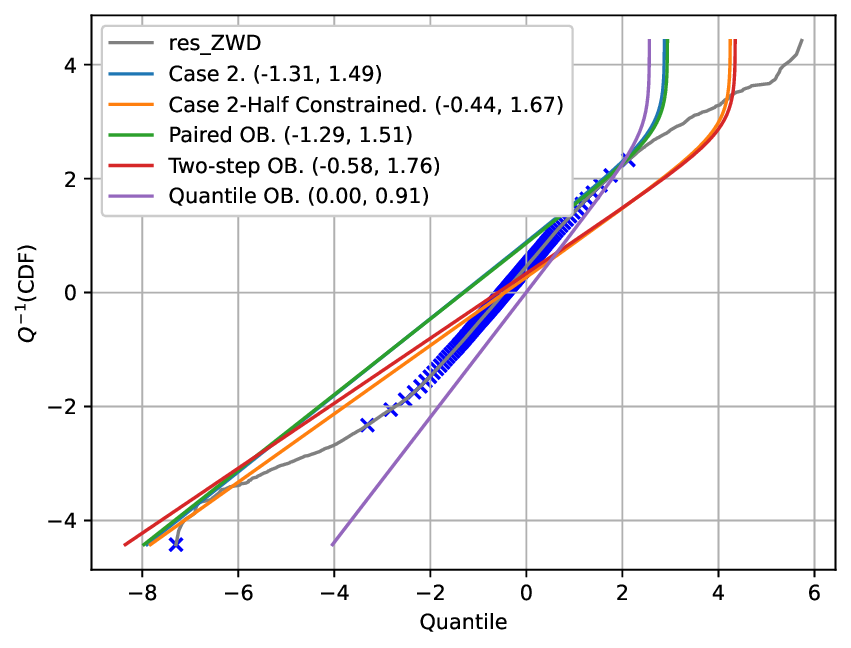}}
    \hfill
    \subfloat[]{\includegraphics[width=0.33\textwidth]{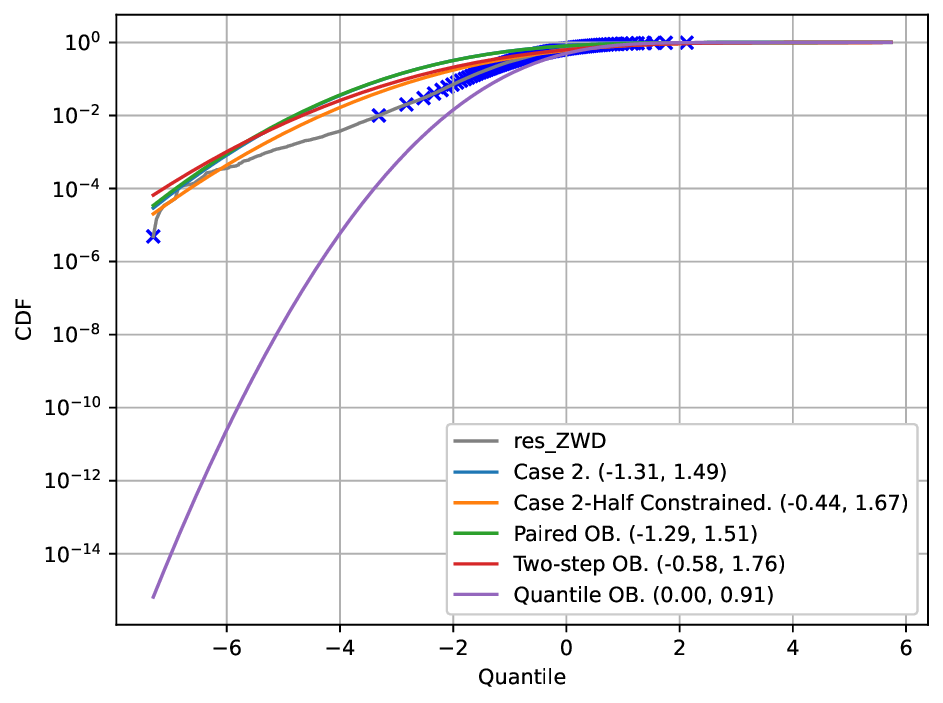}}
    \caption{Visualization of left bounds for \gls{ZWD} residual errors using various methods: 
    (a) PDF comparison, (b) QQ plot, and (c) CDF (log scale).}
    \label{fig:zwd_result}
\end{figure*}

\begin{table*}[!htbp]
    \centering
    \caption{Results on \gls{ZWD} residual errors}
    \label{tab:zwd_result}
    \resizebox{\linewidth}{!}{
\begin{tabular}{l|lllllll|lllllll}
\hline
Method                & $\mu_L$ & $\sigma_L$ & OB & $PL_L^1$ & $PL_L^{10}$ & $W_L$ & $\mathcal{K}_L$ & $\mu_R$ & $\sigma_R$ & OB & $PL_R^1$ & $PL_R^{10}$ & $W_R$ & $\mathcal{K}_R $\\ \hline
Paired OB                   & -1.295 & 1.508   & True         & -5.956 & -2.772 & 59.098 & 1.914  & 0.055  & 1.382   & True         & 4.327 & 1.409  & 39.500 & 1.280 \\ 
Two-step OB                 & -0.582 & 1.757   & True         & -6.012 & -2.303 & 34.131 & 1.156  & -0.430 & 1.532   & True         & 4.306 & 1.071  & 21.695 & 0.727 \\ 
Quantile OB                 & 0.000  & 0.913   & 0-0.98       & -2.822 & -0.894 & 31.199 & -0.616 & 0.000  & 1.424   & 0.9-0.98     & 4.400 & 1.394& 38.524 & 1.262 \\ 
Case 2                  & -1.308 & 1.490   & True         & -5.916 & -2.769 & 59.030 & 1.904  & -0.042 & 1.315   & True         & 4.024 & 1.246  & 31.769 & 0.981 \\ 
Case 2 half constraints & -0.442 & 1.670   & True         & -5.605 & -2.078 & 23.437 & 0.705  & -0.477 & 1.475   & True         & 4.082 & 0.967  & 16.864 & 0.534 \\ \hline
\end{tabular}
}
\end{table*}

Table~\ref{tab:zwd_result} quantitatively confirms these observations. The quantile overbounding method does not achieve conservative bounds, while Case 2 shows marginal improvements over paired overbounding across all metrics, although it slightly lags behind two-step overbounding in some cases. By contrast, the Case 2 half-constrained variant consistently outperforms both two-step and paired overbounding across all evaluation metrics.
Specifically, Case 2 improves $W_L$ and $\mathcal{K}_L$ by 0.1\% and 0.5\% over paired overbounding on the left bound, with more notable gains on the right bound: 20\% and 23\% improvements, respectively. It also reduces the range of $PL^1$ by 3.3\% and 3.7\% compared to paired and two-step overbounding, respectively. For $PL^{10}$, Case 2 shows a 4\% improvement over paired overbounding but performs slightly worse than two-step overbounding.
The Case 2 half-constrained variant shows more substantial benefits. It improves $W_L$ and $\mathcal{K}_L$ by 31\% and 39\% over two-step overbounding for the left bound, while delivering similar improvements of 22\% and 26\% on the right bound. In terms of protection level range, it reduces the range of $PL^1$ and $PL^{10}$ by 6\% and 10\% compared to two-step overbounding.
When compared to the ground truth $PL^1$ and $PL^{10}$ ranges ([–5.252, 3.472] and [–1.610, 0.600], respectively), the Case 2 half-constrained variant further reduces deviation: $PL^1$ total deviation decreases from 1.59 (two-step) and 1.56 (paired) to just 0.96—a reduction of more than 38\%. Similarly, $PL^{10}$ deviation decreases from 1.16 (two-step) and 1.97 (paired) to 0.83, achieving a reduction of over 28\%.
These results suggest that for symmetric and unimodal distributions, applying a symmetric (half-constraint) overbounding strategy is both more appropriate and tighter than enforcing strict, full-range conservatism. 

To visualize the conditional distribution of multipath errors in a unified manner, we normalized the sample distributions using each method’s predicted parameters. Fig.~\ref{fig:mp_result} compares the resulting left-tail normalized distributions across different methods against the standard Gaussian distribution, which represents the ideal overbounding reference. As shown, all normalized sample distribution curves lie below the standard normal curve, indicating that all methods successfully maintain conservatism.
However, because the two-step (green line) and paired overbounding (yellow line) methods are applied to the unconditional error distribution without considering feature-dependent variability, their normalized results appear nearly identical and exhibit noticeable deviations from the Gaussian reference. In contrast, our method (blue line) achieves a much closer match to the standard Gaussian distribution, particularly evident in the logarithmic-scale \gls{CDF} plot (Fig.~\ref{fig:mp_result}(c)). This suggests that by leveraging conditional information, our approach produces tighter, less conservative bounds while remaining conservative in this empirical evaluation.

\begin{figure*}[!htbp]
    \centering
    \subfloat[]{\includegraphics[width = 0.33\textwidth]{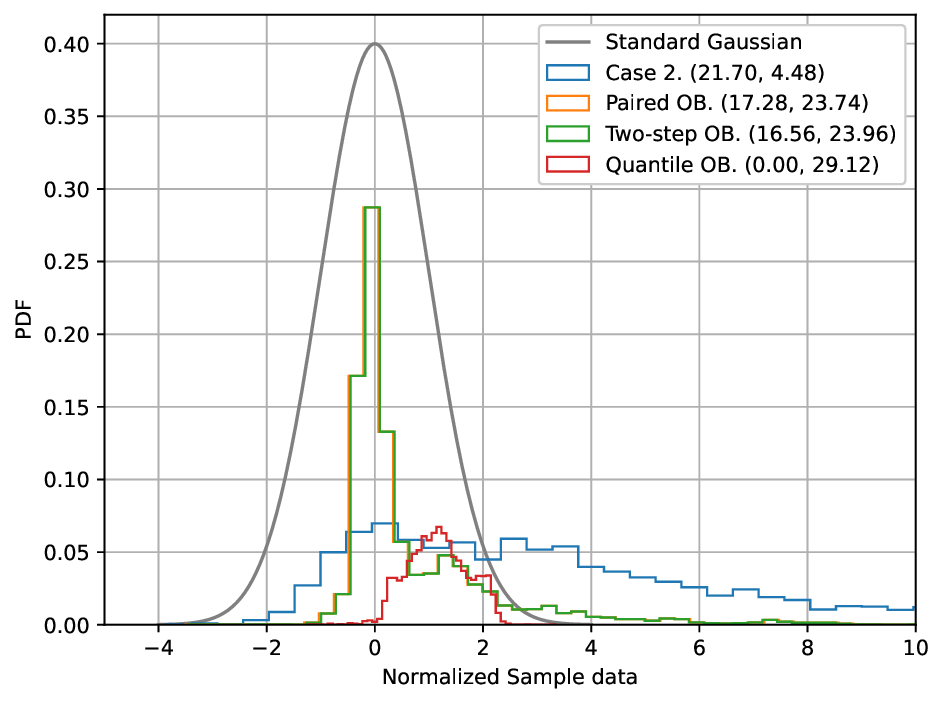}}
    \hfill
    \subfloat[]{\includegraphics[width = 0.33\textwidth]{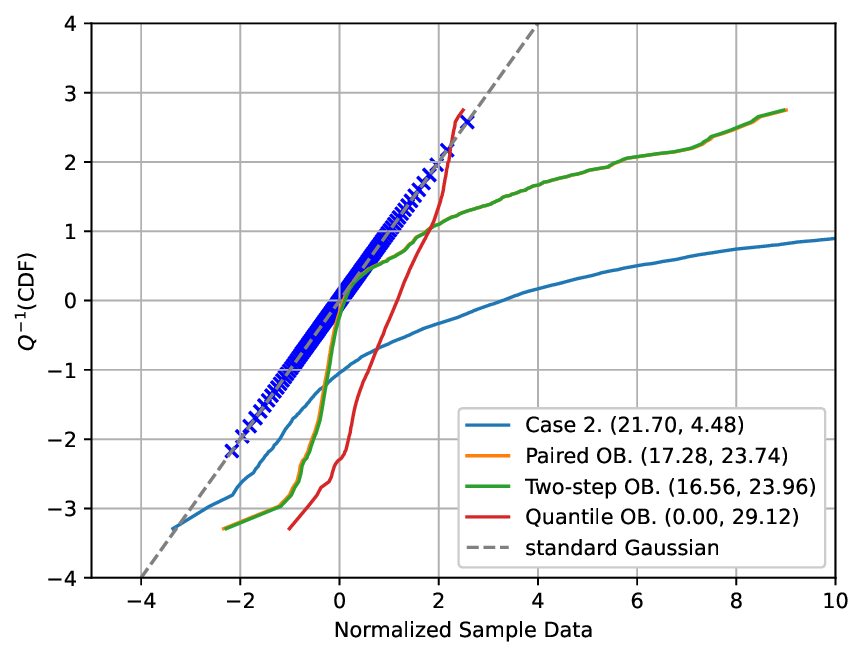}} 
    \hfill
    \subfloat[]{\includegraphics[width = 0.33\textwidth]{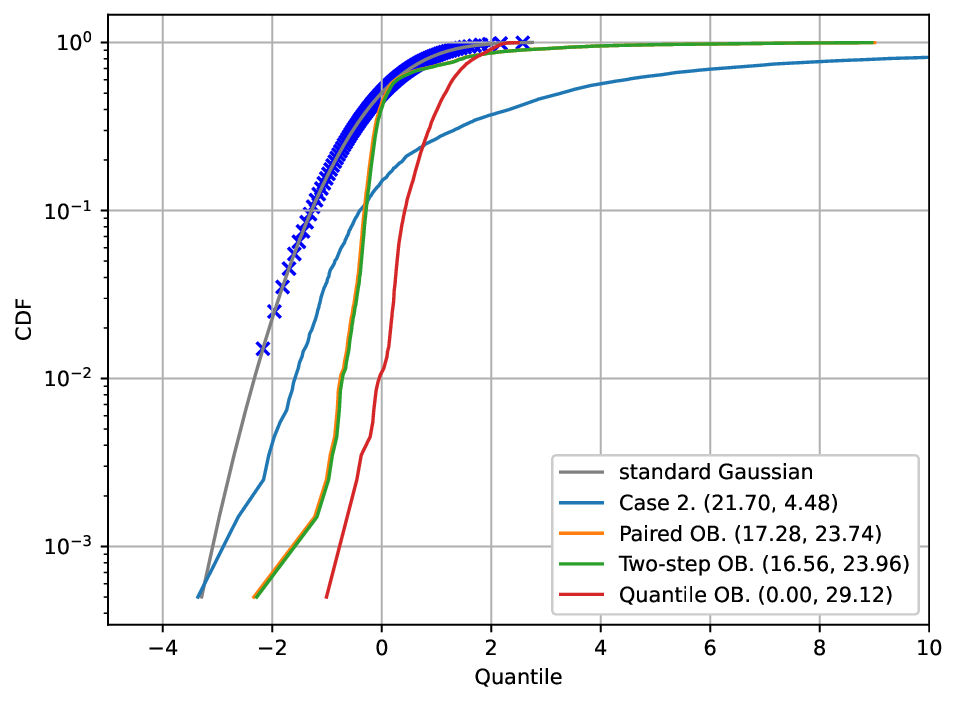}}
    \caption{Visualization of left bounds for multipath errors using various methods: 
    (a) \gls{PDF} comparison, 
    (b) \gls{QQ} plot, and
    (c) \gls{CDF} in logarithmic scale.}
    \label{fig:mp_result}
\end{figure*}

\begin{table*}[!htbp]
    \centering
    \caption{Results on multipath errors}
    \label{tab:mp_result}
    \resizebox{\linewidth}{!}{
\begin{tabular}{l|llllll|llllll}
\hline
Method                & $\mu_L$ & $\sigma_L$ & OB & $PL_L^1$ & $PL_L^{10}$ & $\mathcal{K}_L$ & $\mu_R$ & $\sigma_R$ & OB & $PL_R^1$ & $PL_R^{10}$ & $\mathcal{K}_R $\\ \hline

Paired OB   & 17.281 & 23.738  & True         & -56.092  & -5.971  & 2.610 & 106.311 & 43.823 & True   & 241.767 & 149.239 & 3.466  \\ 
Two-step OB & 16.562 & 23.958  & True         & -57.490  & -6.906  & 2.748 & 112.349 & 47.090 & True   & 257.902 & 158.476 & 3.754  \\ 
Quantile OB & 0.000  & 29.120  & True         & -106.236 & -33.667 & 5.703 & 0.000   & 0.368  & 0-0.98 & 1.139  & 0.361  & -1.000 \\ 
Case 2      & 21.702 & 4.482   & True         & 7.848    & 17.311  & 2.150 & 46.630  & 4.963  & True   & 61.971  & 51.492  & 2.278  \\ \hline
\end{tabular}
}
\end{table*}
Quantitative results summarized in Table~\ref{tab:mp_result} further illustrate the significant improvements offered by our proposed approach, particularly in protection-level performance. Our method substantially reduces the $PL^1$ range ([7.848, 61.971]) compared to two-step overbounding ([-57.49, 257.902]) and paired overbounding ([-56.092, 241.767]), corresponding to reductions of 83\% and 81.5\%, respectively. Similarly, the range of $PL^{10}$ ([17.311, 51.492]) exhibits reductions of 79\% (vs. two-step) and 78\% (vs. paired). Additionally, the average $\sigma$ is reduced by an order of magnitude, from tens down to approximately four, with notable improvements in the average overbounding factor $\mathcal{K}$ (22\% left bound and 39\% right bound vs. two-step; 18\% left bound and 34\% right bound vs. paired). Our method also clearly outperforms quantile overbounding as proposed in~\cite{liu2024overbounding}, reinforcing the method's promise for conditional error distributions.

Fig.~\ref{fig:iono_result} visualizes left side results for ionosphere residual errors using three variants of our method—Case 2 (full constraints using features \gls{Rfit} and \gls{RCM}), Case 2-Half Constrained (symmetric constraints with \gls{Rfit} and \gls{RCM}), and Case 2-Half+Elev (symmetric constraints with \gls{Rfit}, \gls{RCM}, and elevation angle)—alongside traditional methods including \gls{GIVE}, paired, and two-step Gaussian overbounding. Our methods consistently exhibit a more Gaussian-like shape compared to \gls{GIVE} (red line), paired (purple line), and two-step (brown line) methods. The normalized distributions under GIVE resemble a Laplace distribution, while paired and two-step Gaussian methods approximate a Dirac-like function due to using fixed parameters across the entire dataset. Fig.~\ref{fig:iono_result}(b) further confirms these differences, with \gls{QQ} plots for paired and two-step Gaussian methods appearing nearly vertical, indicating substantial deviations from Gaussian normality. In contrast, our method’s variants produce \gls{QQ} plots whose central regions closely parallel the normal distribution line, demonstrating better shape preservation and conditional conservatism.
\begin{figure*}[!htbp]
    \centering
    \subfloat[]{\includegraphics[width=0.33\textwidth]{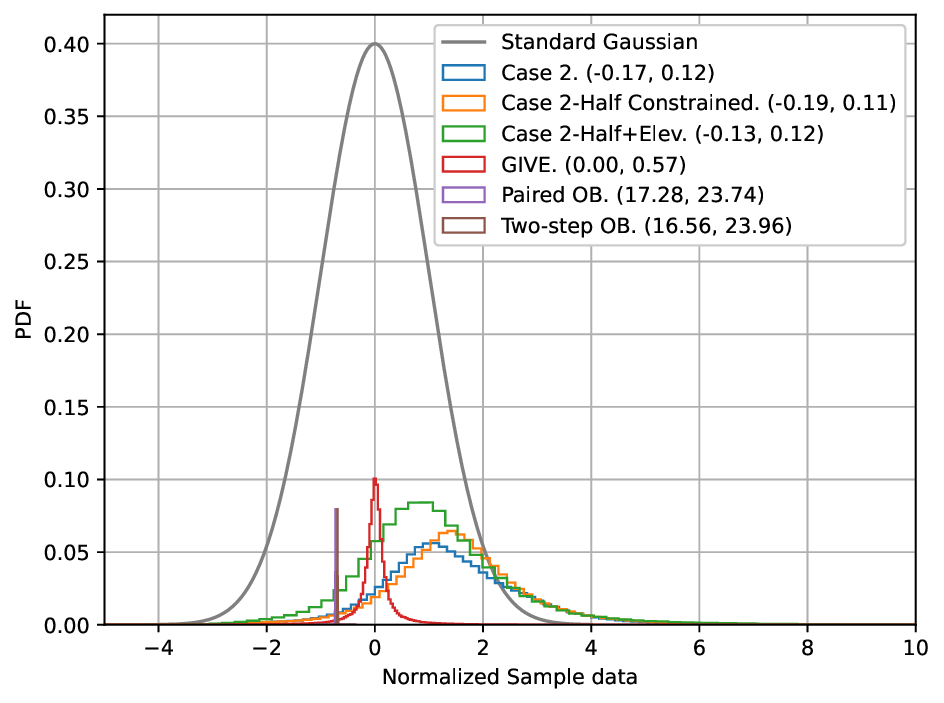}}
    \hfill
    \subfloat[]{\includegraphics[width=0.33\textwidth]{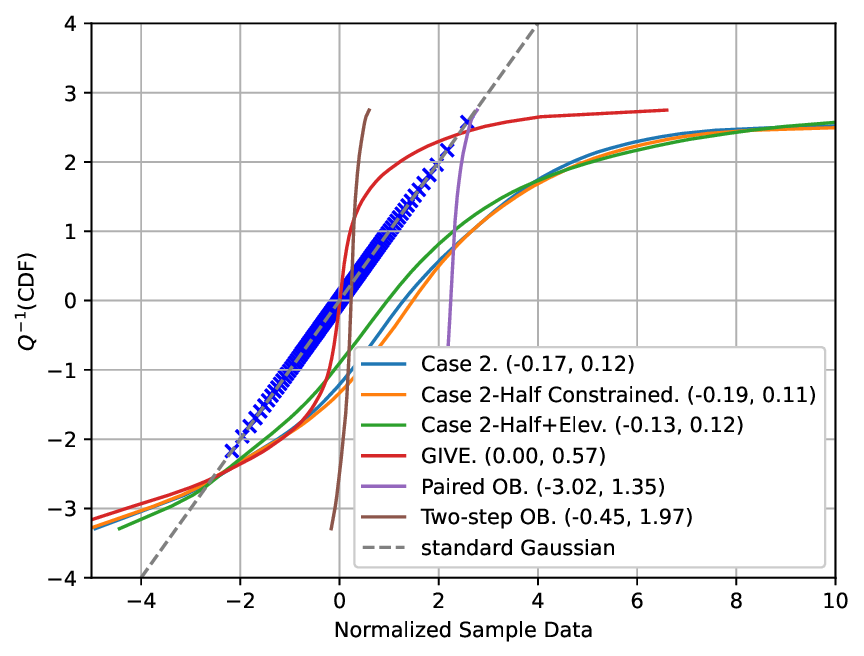}}
    \hfill
    \subfloat[]{\includegraphics[width=0.33\textwidth]{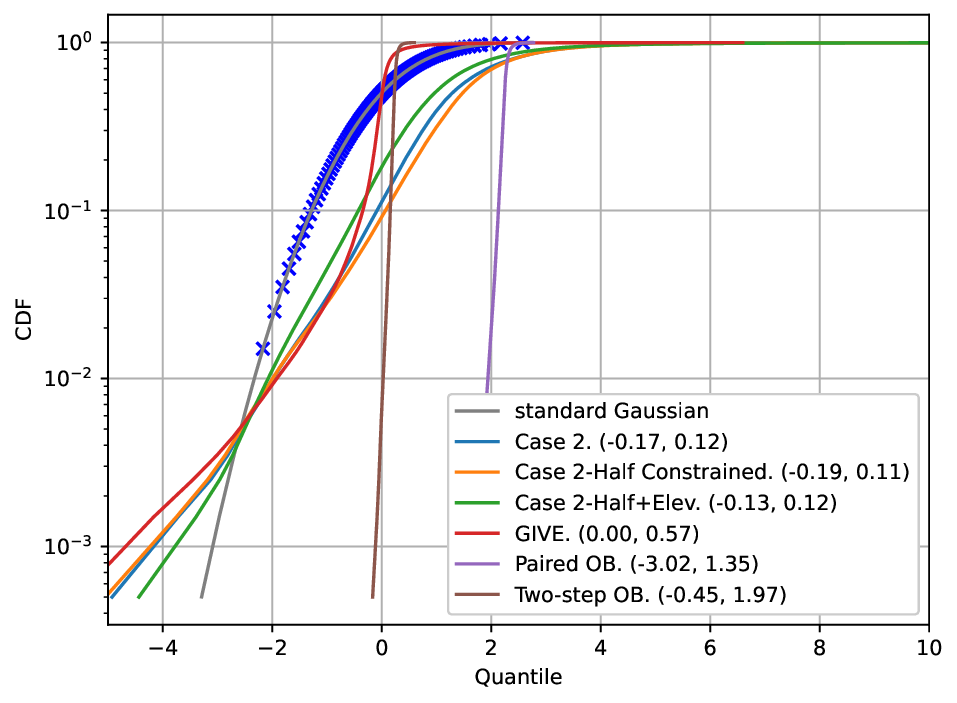}}
    \caption{Visualization of left bounds for ionosphere delay residual errors using various methods: 
    (a) PDF comparison, (b) QQ plot, and (c) CDF (log scale).}
    \label{fig:iono_result}
\end{figure*}

\begin{table*}[!htbp]
    \centering
    \caption{Results on ionosphere delay residual errors}
    \label{tab:iono_result}
    \resizebox{\linewidth}{!}{
\begin{tabular}{l|llllll|lllllll}
\hline
Method                & $\mu_L$ & $\sigma_L$ & OB & $PL_L^1$ & $PL_L^{10}$ & $\mathcal{K}_L$ & $\mu_R$ & $\sigma_R$ & OB & $PL_R^1$ & $PL_R^{10}$ & $\mathcal{K}_R $\\ \hline
Paired OB                                        & -3.020  & 1.350      & True         & -7.193  & -4.343   & 10.518 & 2.580  & 1.197 & True & 6.280 & 3.752 & 10.445 \\ 
Two-step OB                                      & -0.454  & 1.974      & True         & -6.556  & -2.387   & 5.284  & 0.146  & 1.699 & True & 5.396 & 1.809 & 4.519  \\ 
GIVE                                             & 0.000   & 0.573      & True         & -1.772  & -0.562   & 1.288  & 0.000  & 0.573 & True & 1.772 & 0.562 & 1.326  \\ 
Case 2                                       & -0.173  & 0.121      & True         & -0.545  & -0.291   & 2.143  & 0.181 & 0.108 & True & 0.515 & 0.287 & 2.545  \\ 
Case 2-Half Constrained                    & -0.186  & 0.113      & True         & -0.537  & -0.297   & 2.441  & 0.143 & 0.121 & True & 0.518 & 0.262 & 1.808  \\ 
Case 2-Half+Elev & -0.130  & 0.123      & True         & -0.509  & -0.250   & 1.462  & 0.087 & 0.131 & True & 0.491 & 0.215 & 0.988  \\ \hline
\end{tabular}
}
\end{table*}

Table~\ref{tab:iono_result} quantitatively supports these findings. Our three proposed variants substantially improve protection-level metrics. Specifically, the $PL^1$ range (approximately [-0.5, 0.5]) shows improvements exceeding 90\% compared to two-step ([-6.556, 5.396]) and paired overbounding ([-7.193, 6.280]), while the $PL^{10}$ range (approximately [-0.3, 0.3]) reduces by more than 85\% compared to the same baselines. The GIVE approach, despite using conditional distributions, still yields larger protection-level ranges ($PL^1$: [-1.772, 1.772], $PL^{10}$: [-0.562, 0.562]), making our methods approximately 71\% (for $PL^1$) and 47\% (for $PL^{10}$) better. Moreover, average predicted $\sigma$ decreases significantly from around 0.5 (GIVE) to approximately 0.1 in our approaches. Although our methods do not always improve the average overbounding factor $\mathcal{K}$ compared to GIVE, their overall advantages in other critical metrics are evident. Finally, the inclusion of the elevation angle as an additional feature consistently enhances performance across all metrics, indicating that richer feature information enables tighter protection-level estimation while preserving the required conservative constraints.

Overall, these empirical evaluations confirm the practical effectiveness of the proposed method for conditional error distributions in the tested scenarios.

\section{Conclusion\label{sec6}}
This paper studied a safety-oriented uncertainty quantification (UQ) problem: learning conditional predictive uncertainty that is conservative in the tails while remaining as tight as possible to preserve downstream availability. Conventional UQ methods for regression (e.g., Bayesian approximations, ensembles, variance networks, quantile regression, and conformal prediction) primarily target calibration or marginal coverage for a single target variable, and they are not tailored to convolution-based propagation across multiple error sources. In contrast, integrity-style Gaussian overbounding enables convolution-transitive conservatism, but classical constructions are often heuristic, distribution-dependent, and difficult to extend to feature-conditioned settings; moreover, they typically lack an explicit tightness objective.

To bridge these lines, we proposed a single-stage conservative learning framework that outputs a context-conditioned Gaussian overbound and trains it end-to-end under quantile-space overbounding constraints with an excess-mass relaxation. The proposed objective integrates multi-quantile regression with relaxed paired-overbounding constraints, monotonicity regularization, and a Wasserstein-distance-inspired penalty, thereby explicitly controlling the conservatism--tightness trade-off. We further provided a scoped theoretical analysis characterizing finite-grid conservatism, its extension to a certified continuous interval under three explicit regularity assumptions, and local regularity of the objective on compact domains.

Experiments on synthetic mixtures and three real-world error scenarios (troposphere residuals, urban multipath, and ionospheric residuals) demonstrated that embedding overbounding constraints into learning yields consistently tighter conservative bounds than two-step Gaussian overbounding, paired overbounding, and quantile overbounding, while maintaining conservatism on the enforced grid and in the reported empirical evaluations. Across all scenarios, the proposed method reduced system-relevant protection levels while improving distributional proximity (lower Wasserstein discrepancy) and decreasing average overbounding factors. Representative results include substantial reductions in protection-level range for multipath and ionosphere errors (e.g., $>78\%$ and $>85\%$ relative to two-step/paired baselines), together with clear tightness improvements for troposphere residuals. Ablation studies further confirmed that the Wasserstein penalty, quantile reweighting, and the proposed parameterization are critical to achieving a favorable conservatism--tightness balance.

Overall, this work connects quantile-space learning in modern UQ with classical integrity overbounding and provides a principled recipe for learning uncertainty models that support conservative, convolution-based propagation with clearly scoped guarantees. Future work includes extending beyond univariate Gaussians to multivariate bounds, refining certification for more general bounding families, and incorporating temporal structure for sequential uncertainty and change detection.

% \section*{Acknowledgments}
% This work described in this paper was supported by a grant from the Research Grants Council of the Hong Kong Special Administrative Region, China (Project No. 15214523) and the Smart Traffic Fund of the Hong Kong SAR Government (Project No. PSRI/74/2309/PR). The support is gratefully acknowledged. The authors also acknowledge the use of ChatGPT for English proofreading.

\appendices
\numberwithin{equation}{section}
\renewcommand{\theequation}{\thesection\arabic{equation}}

\section{Derivation of the Weighting Scheme for Multiple Quantile Regression \label{app1}}
This appendix gives a heuristic derivation of the weighting scheme used in multiple quantile regression. The goal is to motivate the relative scaling of the weights across quantile levels rather than to establish a unique optimal choice.
\begin{equation}
\label{eq:ap1}
\int_{-\infty}^{a} f(y)\,y\,dy
= aF(a) - \int_{-\infty}^{a} F(y)\,dy .
\end{equation}
\begin{equation}
\label{eq:ap1_Gdef}
G(a) \triangleq \int_{-\infty}^{a} F(u)\,du .
\end{equation}
where $f(y)$ and $F(y)$ denote the \gls{PDF} and \gls{CDF} of the random variable $y$, respectively, and $G(\cdot)$ is defined in \eqref{eq:ap1_Gdef}. Using this identity, we derive the cost function for quantile regression at quantile level $\tau$, denoted as $\mathcal{J}_\tau^{QL}$:
\begin{equation}
\begin{aligned}
\mathcal{J}_\tau^{QL}
= &\int_{-\infty}^{q_\tau} f(y)(1-\tau)(q_\tau-y)\,dy \\
&+ \int_{q_\tau}^{\infty} f(y)\tau(y-q_\tau)\,dy \\
= &(1-\tau)q_\tau F(q_\tau)-\tau q_\tau (1-F(q_\tau)) \\
&- \int_{-\infty}^{q_\tau} f(y)\,y\,dy + \tau\int_{-\infty}^{\infty} f(y)\,y\,dy \\
= &\, q_\tau(F(q_\tau)-\tau)+\tau \mathbb{E}[y]-q_\tau F(q_\tau) +G(q_\tau) \\
= &\,\tau \mathbb{E}[y] -\tau q_\tau +G(q_\tau) .
\end{aligned}
\end{equation}

where $q_\tau$ is the quantile at level $\tau$, and $\mathbb{E}[y]$ denotes the expectation of $y$.
To ensure that each component $w_\tau \mathcal{J}_\tau^{QL}$ is comparable across different values of $\tau$, we assign a weight $w_\tau$ to each $\mathcal{J}_\tau^{QL}$ such that $w_\tau \propto \frac{1}{\mathcal{J}_\tau^{QL}}$. For practical computation, we further approximate $\mathcal{J}_\tau^{QL}$ under the following assumptions:
\begin{itemize}
\item The \gls{CDF} $F(x)$ is approximated by a linear function with slope $\mathrm{slope}$, passing through the point $(q_\tau, \tau)$;
\item The median of $y$ is approximately equal to its expectation, i.e. $\mathbb{E}[y]\simeq F^{-1}(0.5)$.
\end{itemize}
These simplifying assumptions are used only to motivate the dependence of $w_\tau$ on $\tau$.
Under these assumptions, we obtain:
\begin{equation}
\begin{aligned}
\mathcal{J}^{QL}_\tau
&\simeq \tau(\mathbb{E}[y]-q_\tau)+G(q_\tau) \\
&\simeq \tau\left(F^{-1}\left(\frac{1}{2}\right)-q_\tau\right)+\frac{1}{2}\cdot\tau\cdot\frac{\tau}{\mathrm{slope}} \\
&= \tau\left(\frac{1}{2}-\tau\right)\cdot\frac{1}{\mathrm{slope}} + \frac{\tau^2}{2\mathrm{slope}} \\
&= \frac{\tau(1-\tau)}{2\mathrm{slope}} .
\end{aligned}
\end{equation}

This reveals that $\mathcal{J}_\tau^{QL}$ is approximately proportional to $\tau(1-\tau)$. Accordingly, the weighting term can be defined as:
\begin{equation}
w_\tau = \frac{1}{4\tau(1-\tau)}
\end{equation}
where the factor $\frac{1}{4}$ is chosen so that $w_{0.5}=1$. This choice ensures balanced contributions from each quantile level during training, while assigning larger weights to the tails.

\section{Derivation of the Scaling Factor $t$ Condition \label{app2}}
To ensure the conservative overbounding condition in \eqref{eq15} in main text, we require:
\begin{equation}
(1+\epsilon)F_L(\hat{q}_\tau) \geq F(\hat{q}_\tau),\quad \forall \hat{q}_\tau \label{eq:app2_1}
\end{equation}
Under Case 2 and Case 3, the discrete overbounding relation derived in the main manuscript holds:
\begin{equation}
(1+\epsilon)F_L(\hat{q}_\tau)\geq \tau = F(q_\tau) \label{eq:app2_2}
\end{equation}
Therefore, a sufficient condition for \eqref{eq:app2_1} is:
\begin{align}
&F(\hat{q}_\tau) \leq F(q_\tau) =\tau \label{eq:app2_3}\\
\iff\ &\hat{q}_\tau\leq q_\tau \label{eq:app2_4}
\end{align}
That is, ensuring \eqref{eq:app2_3} or \eqref{eq:app2_4} guarantees conservatism.

\subsection{Subgradient Analysis of the Overbounding Loss}
We now derive first-order expressions for the overbounding loss. For the differentiable quantile term $\mathcal{J}^{QL}_\tau$ (and hence $\mathcal{J}^{MQL}$), the Clarke subgradient is a singleton and coincides with the usual gradient.
\par
We start from the integral form of $\mathcal{J}^{QL}_\tau$:
\begin{equation}
\begin{aligned}
\mathcal{J}^{QL}_\tau(\theta)
&=(1-\tau)\int_{-\infty}^{\theta}(\theta-y)\,dF(y)
 +\tau\int_{\theta}^{\infty}(y-\theta)\,dF(y).
\end{aligned}
\end{equation}
The corresponding singleton Clarke subgradient with respect to $\theta$ is
\begin{equation}
\begin{aligned}
\partial_{\theta}\mathcal{J}^{QL}_\tau(\theta)
&=(1-\tau)\int_{-\infty}^{\theta} dF(y)
 - \tau\int_{\theta}^{\infty} dF(y) \\
&=(1-\tau)F(\theta)-\tau(1-F(\theta)) \\
&= F(\theta) - \tau.
\end{aligned}
\end{equation}
For modified quantile loss $\mathcal{J}^{MQL}$ (substituting $t\tau$ for $\tau$), this yields the following singleton Clarke subgradients:
\begin{equation}
    \partial_{\hat{q}_\tau}\mathcal{J}^{MQL} = w_\tau(F(\hat{q}_\tau)-t\tau)
\end{equation}
and the corresponding parameter-wise expressions are:

\begin{align}
\partial_{\mu_\tau} \mathcal{J}^{MQL} &= w_\tau (F(\hat{q}_\tau) - t\tau) \\
\partial_{\sigma_\tau} \mathcal{J}^{MQL} &= w_\tau Q^{-1}\left(k_\tau \tfrac{\tau}{1+\epsilon} \right) (F(\hat{q}_\tau) - t\tau) \\
\partial_{k_\tau} \mathcal{J}^{MQL} &= w_\tau \sigma_\tau \left(F(\hat{q}_\tau) - t\tau \right) \left(Q^{-1}\left(k_\tau \tfrac{\tau}{1+\epsilon} \right)\right)' \tfrac{\tau}{1+\epsilon}
\end{align}

For $\mathcal{J}_{p}$, define
\begin{equation}
    s_\tau := \mu_\tau + Q^{-1}\left(k_\tau \tfrac{\tau}{1+\epsilon}\right)\sigma_\tau
    - \left(\mu + Q^{-1}\left(\tfrac{\tau}{1+\epsilon}\right)\sigma\right),
    \qquad \tau<\tfrac12.
\end{equation}
On the conservative branch enforced by \eqref{eq31}, we have $s_\tau\ge 0$, so $|s_\tau|=s_\tau$. More generally, the expressions below can be interpreted as admissible Clarke subgradients. Let $n_L$ be the number of quantile levels satisfying $\tau<\tfrac{1}{2}$, and let $n_{\min}$ be the number of left-tail quantiles sharing the minimal $\mu$. Then an equal-split subgradient choice for $\mu=\min(\mu_\tau)$ gives:

\begin{align}
\partial_{\mu_\tau} \mathcal{J}_p &=
\begin{cases}
1 - \tfrac{n_L}{n_{\text{min}}} & \text{if } \mu = \mu_\tau \\
1 & \text{otherwise}
\end{cases} \\
\partial_{\sigma_\tau} \mathcal{J}_p &= Q^{-1}\left(k_\tau \tfrac{\tau}{1+\epsilon}\right) - Q^{-1}\left(\tfrac{\tau}{1+\epsilon}\right) \\
\partial_{k_\tau} \mathcal{J}_p &= \sigma_\tau \left(Q^{-1}\left(k_\tau \tfrac{\tau}{1+\epsilon} \right)\right)' \tfrac{\tau}{1+\epsilon}
\end{align}

For $\tau<\tfrac{1}{2}$, a convenient branchwise subgradient selection for the total cost function is given below. Near a stationary solution the monotonicity penalty $\mathcal{J}_m$ is expected to be small, so we omit its contribution here for readability:

\begin{align}
\partial_{\mu_\tau} \mathcal{J} &=
\begin{cases}
w_\tau (F(\hat{q}_\tau) - t\tau) + \lambda \left(1 - \tfrac{n_L}{n_{\text{min}}}\right), & \mu = \mu_\tau, \\
w_\tau (F(\hat{q}_\tau) - t\tau) + \lambda, & \text{otherwise}
\end{cases} \\
\partial_{\sigma_\tau} \mathcal{J} &= w_\tau Q^{-1}\left(k_\tau \tfrac{\tau}{1+\epsilon}\right) (F(\hat{q}_\tau) - t\tau) \\
&\quad + \lambda \left(Q^{-1}\left(k_\tau \tfrac{\tau}{1+\epsilon}\right) - Q^{-1}\left(\tfrac{\tau}{1+\epsilon}\right)\right) \nonumber \\
\partial_{k_\tau} \mathcal{J} &= \left[ w_\tau (F(\hat{q}_\tau) - t\tau) + \lambda \right] \cdot \sigma_\tau \left(Q^{-1}\left(k_\tau \tfrac{\tau}{1+\epsilon} \right)\right)' \tfrac{\tau}{1+\epsilon}
\end{align}

\subsection{Case 2 Derivation for scaling factor $t$}
For Case 2, with shared $\sigma$ and $k$, we compute: 

\begin{align}
\partial_{\sigma} \mathcal{J} &= \sum_{\tau \in \mathcal{T}} w_\tau Q^{-1}\left(k \tfrac{\tau}{1+\epsilon}\right) (F(\hat{q}_\tau) - t\tau) \\
&\quad + \lambda \sum_{\tau < \frac{1}{2}} \left(Q^{-1}\left(k \tfrac{\tau}{1+\epsilon}\right) - Q^{-1}\left(\tfrac{\tau}{1+\epsilon}\right)\right) \nonumber \\
\partial_{k} \mathcal{J} &= \sum_{\tau \in \mathcal{T}} w_\tau \sigma (F(\hat{q}_\tau) - t\tau) \left(Q^{-1}\left(k \tfrac{\tau}{1+\epsilon} \right)\right)' \tfrac{\tau}{1+\epsilon} \\
&\quad + \lambda \sum_{\tau < \frac{1}{2}} \sigma \left(Q^{-1}\left(k \tfrac{\tau}{1+\epsilon} \right)\right)' \tfrac{\tau}{1+\epsilon} \nonumber
\end{align}

Due to shared $\sigma$ and $k$, each $\mu_\tau$ has more degrees of freedom in adjusting the location of its corresponding $q_\tau$. Once $\sigma$ and $k$ are close to a first-order stationary regime, $\mu_\tau$ can self-adjust to optimize the objective. Therefore, it is sufficient to focus on the branchwise stationarity inclusion $0 \in \partial_{\mu_\tau}\mathcal{J}$. Selecting the branchwise subgradient above and setting it to zero, we obtain:

\begin{align}
F(\hat{q}_\tau) &=
\begin{cases}
t\tau - \tfrac{\lambda}{w_\tau}(1 - \tfrac{n_L}{n_{\text{min}}}) & \mu = \mu_\tau \\
t\tau - \tfrac{\lambda}{w_\tau} & \text{otherwise}
\end{cases}
\end{align}

To satisfy the overbounding condition in \eqref{eq:app2_3}, we require:
\begin{align}
    &F(\hat{q}_\tau)\leq\tau \\
    \implies &t\leq 1 -\left(\frac{n_L}{n_{\text{min}}}-1\right) \frac{\lambda}{w_\tau \tau}
\end{align}
The most restrictive case occurs when $n_{\min}=1$, which yields
\begin{align}
    t\leq 1 -(n_L-1) \frac{\lambda}{w_\tau \tau}\label{eq:app2_88}
\end{align}
as the worst-case sufficient bound used in the main text. For example, if using 100 quantiles with $\epsilon=0.0025$ (as adopted in~\cite{blanch2018gaussian}), then $n_L=49$, and the most conservative bound is $t\leq 1-192\lambda$. In practice we use the slightly smaller value $t=1-200\lambda$ as a heuristic finite-grid safety margin rather than as a formal certification statement.

\subsection{Case 3 Derivation for scaling factor $t$}
In Case~3, the Gaussian parameters $\mu$ and $\sigma$ are shared across all quantile levels, while the auxiliary factors $k_\tau$ are optimized independently for each $\tau$.

When $\sigma$ and $\mu$ are fixed, we analyze the implication of the branchwise stationarity inclusion $0 \in \partial_{k_\tau}\mathcal{J}$ on conservatism, using the same conservative branch as above. Since $\sigma>0$ and $\tfrac{\tau}{1+\epsilon}>0$ for $\tau\in(0,1)$,  and recognizing that the derivative of the quantile function satisfies $(Q^{-1}(u))' = 1/\phi(Q^{-1}(u))>0$ for $u\in(0,1)$, where $\phi(z) = \frac{1}{\sqrt{2\pi}}e^{-\frac{z^2}{2}}$ is the standard normal \gls{PDF}, selecting that branchwise subgradient and setting it to zero yields:
\begin{align}
    F(\hat{q}_\tau) = t\tau - \frac{\lambda}{w_\tau}
    \le \tau,
    \label{eq:app2_89}\\
    \implies\quad
    t \le 1 + \frac{\lambda}{w_\tau \tau}.
    \label{eq:app2_90}
\end{align}

Therefore, at the level of first-order stationarity, choosing $t=1$ leads to $F(\hat q_\tau)=\tau-\lambda/w_\tau\le\tau$, and is conservative.

However, in practice, optimizing $k_\tau$ is numerically unstable in extreme quantile regions.
As $u\to 0$ or $u\to 1$, $\phi(Q^{-1}(u))\to 0$ and thus $(Q^{-1}(u))'$ becomes extremely large, which can induce exploding first-order sensitivities.
Meanwhile, $k_\tau$ is parameterized (and effectively bounded) via a sigmoid, whose derivative vanishes in saturated regimes.
The interaction between the exploding sensitivity of $Q^{-1}$ and the saturating sigmoid often prevents the optimizer from reaching the ideal stationary point or pushes solutions to the boundary of $k_\tau$.
As a result, the theoretical relation in \eqref{eq:app2_89} may not hold numerically, and some quantile levels may exhibit $F(\hat q_\tau)>\tau$ (underbounding), violating the desired overbounding property.
Consequently, Case~3 is useful as a theoretical comparison but does not reliably guarantee conservatism in practice and is therefore not adopted as the main method.

\subsection{Numerical Stability for Extreme Tail Probabilities}
During training, extreme quantile levels can make the optimization stiff because $(Q^{-1}(u))' = 1/\phi(Q^{-1}(u))$ grows rapidly as $u\to 0$ or $u\to 1$. For example, $(Q^{-1}(10^{-7}))' \approx 1.86 \times 10^6$, so even modest perturbations in the argument of $Q^{-1}$ can produce very large gradients. This interacts unfavorably with the sigmoid parameterization of $k$ in saturated regimes, which motivates the practical clamping and compact-domain analysis used in the main text.

\section{Continuous-Interval Extension of Finite-Grid Conservatism \label{app3}}
This appendix supplements Section~\ref{sec3} with a discrete-to-continuous extension theorem for the left-tail conservatism constraint.

To make precise what is certified beyond the finite grid constraint in Section~\ref{sec3_2}, we restrict attention to a left-tail certified interval $\mathcal{I}_L:=[\tau_{\min},\tfrac12]$ with $\tau_{\min}>0$, and to the ordered left-tail grid $\mathcal{T}_L:=\{\tau_1<\cdots<\tau_m\}\subset \mathcal{I}_L$. Let
\begin{equation}
    h_{\mathcal{T}}:= \max_{1\le i<m} (\tau_{i+1}-\tau_i)
\end{equation}
denote the maximum grid spacing, and define the quantile-gap function
\begin{equation}
    g(\tau):=F^{-1}(\tau)-F_L^{-1}\!\left(\frac{\tau}{1+\epsilon}\right), \quad \tau\in\mathcal{I}_L.
\end{equation}

\begin{assumption}[Tail regularity]
For each fixed context $x$, the conditional \gls{CDF} $F(\cdot|x)$ is continuous and strictly increasing on the certified left-tail region induced by $\mathcal{I}_L$, so that $F^{-1}(\tau|x)$ is well-defined for $\tau\in\mathcal{I}_L$.
\end{assumption}

\begin{assumption}[Certified interval and grid]
The certified left-tail interval is $\mathcal{I}_L=[\tau_{\min},\tfrac12]$ with $\tau_{\min}>0$, and the enforced quantile grid $\mathcal{T}_L\subset\mathcal{I}_L$ is finite with maximum spacing $h_{\mathcal{T}}$.
\end{assumption}

\begin{assumption}[Quantile regularity]
The actual quantile function $q(\tau):=F^{-1}(\tau)$ and the Gaussian bound quantile function $\bar q(\tau):=F_L^{-1}(\tfrac{\tau}{1+\epsilon})$ are Lipschitz continuous on $\mathcal{I}_L$ with constants $L_q$ and $L_{\bar q}$, respectively.
\end{assumption}

\begin{theorem}[Grid-to-Continuum Conservatism]
\label{thm:grid_to_continuum}
Suppose Assumptions 1--3 hold and define the discrete left-tail margin
\begin{equation}
    m_{\mathcal{T}}:=\min_{\tau_i\in\mathcal{T}_L} g(\tau_i).
\end{equation}
Then for every $\tau\in\mathcal{I}_L$,
\begin{equation}
    g(\tau)\ge m_{\mathcal{T}}-(L_q+L_{\bar q})h_{\mathcal{T}}.
\end{equation}
Consequently, if $m_{\mathcal{T}}>(L_q+L_{\bar q})h_{\mathcal{T}}$, then
\begin{equation}
    F^{-1}(\tau)\ge F_L^{-1}\!\left(\frac{\tau}{1+\epsilon}\right), \quad \forall \tau\in\mathcal{I}_L,
\end{equation}
so the discrete conservatism constraint extends to the entire certified interval. Moreover, if $h_{\mathcal{T}}\to 0$ along a sequence of refined grids and the discrete margin remains bounded away from zero, then
\begin{equation}
    \sup_{\tau\in\mathcal{I}_L}[-g(\tau)]_+\to 0.
\end{equation}
\end{theorem}
\begin{proof}
For any $\tau\in\mathcal{I}_L$, choose a grid point $\tau_i\in\mathcal{T}_L$ from the same grid cell, so that $|\tau-\tau_i|\le h_{\mathcal{T}}$. By Assumption 3 and the triangle inequality,
\begin{equation}
    |g(\tau)-g(\tau_i)|
    \le |q(\tau)-q(\tau_i)| + |\bar q(\tau)-\bar q(\tau_i)|
    \le (L_q+L_{\bar q})h_{\mathcal{T}}.
\end{equation}
Hence $g(\tau)\ge g(\tau_i)-(L_q+L_{\bar q})h_{\mathcal{T}}\ge m_{\mathcal{T}}-(L_q+L_{\bar q})h_{\mathcal{T}}$. The interval guarantee follows immediately when the right-hand side is positive. The convergence statement follows from the bound
\begin{equation}
    \sup_{\tau\in\mathcal{I}_L}[-g(\tau)]_+\le \bigl[(L_q+L_{\bar q})h_{\mathcal{T}}-m_{\mathcal{T}}\bigr]_+.
\end{equation}
\end{proof}

\section{Compact-Domain Objective Regularity \label{app4}}
This appendix supplements Section~\ref{sec3} with a scoped regularity statement for the Case~2 objective.

\begin{proposition}[Local Regularity of the Objective]
\label{prop:regularity}
Consider Case~2 on a finite quantile grid $\mathcal{T}\subset[\tau_{\min},1-\tau_{\min}]$ with $\tau_{\min}>0$. Assume the optimization is restricted to a compact parameter domain with $\sigma\in[\sigma_{\min},\sigma_{\max}]$ for some $\sigma_{\min}>0$, $s_k\in[s_{\min},s_{\max}]$, and finite $\mu_\tau$. Then the mapping from model outputs to each predicted quantile $\hat q_\tau$ is continuously differentiable, and the full objective \eqref{eq:full_loss} is locally Lipschitz, differentiable almost everywhere, and admits Clarke subgradients.
\end{proposition}
\begin{proof}
Under the stated bounds, the argument $k\tfrac{\tau}{1+\epsilon}$ of $Q^{-1}$ stays inside a compact subset of $(0,1)$ because $\tau$ is bounded away from $0$ and $1$, and $k=1+\epsilon\,\mathrm{sigmoid}(s_k)$ is bounded away from the endpoints of its admissible interval. Therefore $Q^{-1}$ and its derivative are bounded and smooth on the relevant domain, so each $\hat q_\tau$ is a $C^1$ function of the model outputs. The weighted pinball term is piecewise affine in $\hat q_\tau$, the Wasserstein penalty uses absolute values, and the monotonicity term uses $\max(\cdot,0)$; all are locally Lipschitz. Finite sums and compositions of locally Lipschitz functions remain locally Lipschitz, and Rademacher's theorem implies differentiability almost everywhere. Clarke subgradients therefore exist throughout the compact domain.
\end{proof}

Proposition~\ref{prop:regularity} is a regularity statement, not a global optimizer theorem. Its role is to justify first-order stationarity analysis for the compact-domain problem considered in Appendix~\ref{app2}; it does not imply that Adam globally converges for the full nonconvex training problem.

\bibliographystyle{IEEEtran}
\bibliography{References}

\begin{IEEEbiographynophoto}{Ruirui Liu}
    received the B.Eng. degree in Electronic Information Engineering from the Department of Electronic Information Engineering, Huazhong University of Science and Technology, Wuhan, China, in 2016, and the M.Eng. Degree in Electrical and Computer Engineering from the Department of Electrical and Computer Engineering, University of Michigan, in 2018. He has worked for several years in the autonomous driving industry. He is currently working toward the Ph.D. Degree in Aeronautical and Aviation Engineering from The Hong Kong Polytechnic University, Hong Kong. His research interests include the integrity of autonomous driving localization and the application of deep learning techniques to enhance positioning.
\end{IEEEbiographynophoto}
\vspace{11pt}
\begin{IEEEbiographynophoto}{Xuejie Hou}
    received both the B.E. degree and the master's degree in Control Engineering from the Department of Automation Engineering, University of Electronic Science and Technology of China, in 2014 and 2021, respectively. She is currently pursuing a Ph.D. degree at the Hong Kong Polytechnic University. Her research interests encompass tropospheric modeling and integrity monitoring in GNSS.
\end{IEEEbiographynophoto}
\vspace{11pt}
\begin{IEEEbiographynophoto}{Yiping Jiang}
    is an assistant professor at the department of aeronautical and aviation engineering, the Hong Kong polytechnic university. She obtained her Ph.D Degree from the university of new south wales, Sydney, Australia in 2014. Her research interests include precise positioning and integrity monitoring for civil aviation and intelligent transportation systems. Currently, she lectures on avionics system and satellite navigation for undergraduate and postgraduate students.
\end{IEEEbiographynophoto}
\vspace{11pt}
\begin{IEEEbiographynophoto}{Hui Ren}
    received the master's degree in mechanical engineering from the Hong Kong University of Science and Technology, Hong Kong, China, in 2022. He is currently pursuing a Ph.D. degree at The Hong Kong Polytechnic University. His research interests include ionospheric modeling and satellite-based augmentation systems (SBAS) integrity monitoring.

\end{IEEEbiographynophoto}

\vfill

\end{document}